\documentclass[10pt,journal,compsoc]{IEEEtran}
\usepackage{lineno,hyperref}
\usepackage{times}
\usepackage{helvet}
\usepackage{courier}
\usepackage{longtable}% for long tables
\usepackage{url}
\usepackage{graphicx}
\usepackage{subfigure}
\usepackage{amsmath}
\usepackage{amsthm}
\usepackage{amssymb}
\usepackage{float}
\usepackage{fancyhdr}
\usepackage{mathrsfs}
\usepackage{epsfig}
\usepackage{graphics}
\usepackage{epsf}
\usepackage{amssymb}
\usepackage{multirow}
\usepackage{paralist}
\usepackage{algorithm}
\usepackage{algorithmic}
\usepackage{color}

\newtheorem{thm}{Theorem}

\newtheorem{lemma}{Lemma}

\def \x {\mathbf{x}}
\def \g {\mathbf{g}}
\def \D {\mathcal{D}}
\def \u {\mathbf{u}}
\def \w {\mathbf{w}}

\def \L {\mathcal{L}}
\def \c {\mathbf{c}}
\def \S {\mathcal{S}}

\ifCLASSOPTIONcompsoc
  \usepackage[nocompress]{cite}
\else
  \usepackage{cite}
\fi

\ifCLASSINFOpdf

\else

\fi

\begin{document}

\title{Adaptive Subgradient Methods for Online AUC Maximization}

\author{Yi Ding, Peilin Zhao, Steven C.H. Hoi, Yew-Soon Ong
	\IEEEcompsocitemizethanks{\IEEEcompsocthanksitem Yi Ding is with the Department
		of Computer Science, The University of Chicago, Chicago, IL, USA, 60637, E-mail: dingy@uchicago.edu. \protect\\
		\IEEEcompsocthanksitem Corresponding author: Steven C.H. Hoi is with the
		School of Information Systems, Singapore Management University, Singapore
		178902, E-mail: chhoi@smu.edu.sg. \protect\\
		\IEEEcompsocthanksitem Peilin Zhao is with the Data Analytics Department, Institute for Infocomm Research, A*STAR, Singapore 138632, E-mail: zhaop@i2r.a-star.edu.sg.\protect\\
		\IEEEcompsocthanksitem Yew-Soon Ong is with the School of Computer Engineering, Nanyang Technological University, Singapore 639798, E-mail: ASYSOng@ntu.edu.sg.}}

%\markboth{Journal of \LaTeX\ Class Files,~Vol.~14, No.~8, August~2015}
%{Shell \MakeLowercase{\textit{et al.}}: Bare Demo of IEEEtran.cls for Computer Society Journals}

\IEEEtitleabstractindextext{
\begin{abstract}
Learning for maximizing AUC performance is an important research problem in Machine Learning and Artificial Intelligence. Unlike traditional batch learning methods for maximizing AUC which often suffer from poor scalability, recent years have witnessed some emerging studies that attempt to maximize AUC by single-pass online learning approaches. Despite their encouraging results reported, the existing online AUC maximization algorithms often adopt simple online gradient descent approaches that fail to exploit the geometrical knowledge of the data observed during the online learning process, and thus could suffer from relatively larger regret. To address the above limitation, in this work, we explore a novel algorithm of Adaptive Online AUC Maximization (AdaOAM) which employs an adaptive gradient method that exploits the knowledge of historical gradients to perform more informative online learning. The new adaptive updating strategy of the AdaOAM is less sensitive to the parameter settings and maintains the same time complexity as previous non-adaptive counterparts. Additionally, we extend the algorithm to handle high-dimensional sparse data (SAdaOAM) and address sparsity in the solution by performing lazy gradient updating. We analyze the theoretical bounds and evaluate their empirical performance on various types of data sets. The encouraging empirical results obtained clearly highlighted the effectiveness and efficiency of the proposed algorithms.
\end{abstract}
}

\maketitle

%\IEEEdisplaynontitleabstractindextext

%\IEEEpeerreviewmaketitle

\IEEEraisesectionheading{\section{Introduction}\label{sec:introduction}}

\IEEEPARstart{A}{UC} (Area Under ROC curve)~\cite{Fawcett:2006:IRA:1159473.1159475} is an important measure for characterizing machine learning performances in many real-world applications, such as ranking, and anomaly detection tasks, especially when misclassification costs are unknown. In general, AUC measures the probability for a randomly drawn positive instance to have a higher decision value than a randomly sample negative instance. Many efforts have been devoted recently to developing efficient AUC optimization algorithms for both batch and online learning tasks~\cite{DBLP:conf/nips/CortesM03,DBLP:conf/pkdd/CaldersJ07,DBLP:conf/icml/Joachims05,DBLP:journals/jmlr/RudinS09,DBLP:conf/icml/ZhaoHJY11,DBLP:conf/icml/GaoJZZ13}.

Due to its high efficiency and scalability in real-world applications, online AUC optimization for streaming data has been actively studied in the research community in recent years. The key challenge for AUC optimization in online setting is that AUC is a metric represented by the sum of pairwise losses between instances from different classes, which makes conventional online learning algorithms unsuitable for direct use in many real world scenarios. To address this challenge, two core types of Online AUC Maximization (OAM) frameworks have been proposed recently. The first framework is based on the idea of buffer sampling~\cite{DBLP:conf/icml/ZhaoHJY11,DBLP:conf/icml/KarS0K13}, which stores some randomly sampled historical examples in a buffer to represent the observed data for calculating the pairwise loss functions. The other framework focuses on one-pass AUC optimization~\cite{DBLP:conf/icml/GaoJZZ13}, where the algorithm scan through the training data only once. The benefit of one-pass AUC optimization lies in the use of squared loss to represent the AUC loss function while providing proofs on its consistency with the AUC measure~\cite{gao2012consistency}.

Although these algorithms have been shown to be capable of achieving fairly good AUC performances, they share a common trait of employing the online gradient descent technique, which fail to take advantage of the geometrical property of the data observed from the online learning process, while recent studies have shown the importance of exploiting this information for online optimization~\cite{DBLP:journals/jmlr/DuchiHS11}. To overcome the limitation of the existing works, we propose a novel framework of Adaptive Online AUC maximization (AdaOAM), which considers the adaptive gradient optimization technique for exploiting the geometric property of the observed data to accelerate online AUC maximization tasks. Specifically, the technique is motivated by a simple intuition, that is, the frequently occurring features in online learning process should be assigned with low learning rates while the rarely occurring features should be given high learning rates. To achieve this purpose, we propose the AdaOAM algorithm by adopting the adaptive gradient updating framework proposed by~\cite{DBLP:journals/jmlr/DuchiHS11} to control the learning rates for different features. We theoretically prove that the regret bound of the proposed algorithm is better than those of the existing non-adaptive algorithms. We also empirically compared the proposed algorithm with several state-of-the-art online AUC optimization algorithms on both benchmark datasets and real-world online anomaly detection datasets. The promising results validate the effectiveness and efficiency of the proposed AdaOAM.

To further handle high-dimensional sparse tasks in practice, we investigate an extension of the AdaOAM method, which is labeled here as the Sparse AdaOAM method (SAdaOAM). The motivation is that because the regular AdaOAM algorithm assumes every feature is relevant and thus most of the weights for corresponding features are often non-zero, which leads to redundancy and low efficiency when rare features are informative for high dimension tasks in practice. To make AdaOAM more suitable for such cases, the SAdaOAM algorithm is proposed by inducing sparsity in the learning weights using adaptive proximal online gradient descent. To the best of our knowledge, this is the first effort to address the problem of keeping the online model sparse in online AUC maximization task. Moreover, we have theoretically analyzed this algorithm, and empirically evaluated it on an extensive set of real-world public datasets, compared with several state-of-the-art online AUC maximization algorithms. Promising results have been obtained that validate the effectiveness and efficacy of the proposed SAdaOAM.

The rest of this paper is organized as follows. We first review the related works from three core areas: online learning, AUC maximization, and sparse online learning, respectively. Then, we present the formulations of the proposed approaches for handling both regular and high-dimensional sparse data, and their theoretical analysis; we further show and discuss the comprehensive experimental results, the sensitivity of the parameters, and tradeoffs between the level of sparsity and AUC performances. Finally, we conclude the paper with a brief summary of the present work.

\section{Related Work}

Our work is closely related to three topics in the context of machine learning, namely, online learning, AUC maximization, and sparse online learning. Below we briefly review some of the important related work in these areas.

\textbf{Online Learning.} Online learning has been extensively studied in the machine learning communities~\cite{DBLP:books/daglib/0016248,DBLP:journals/jmlr/CrammerDKSS06,DBLP:journals/jmlr/ZhaoHJ11,DBLP:journals/ml/HoiJZY13,DBLP:journals/ai/ZhaoHWL14}, mainly due to its high efficiency and scalability to large-scale learning tasks. Different from conventional batch learning methods that assume all training instances are available prior to the learning phase, online learning considers one instance each time to update the model sequentially and iteratively. Therefore, online learning is ideally appropriate for tasks in which data arrives sequentially. A number of first-order algorithms have been proposed including the well-known Perceptron algorithm~\cite{rosenblatt1958perceptron} and the Passive-Aggressive (PA) algorithm~\cite{DBLP:journals/jmlr/CrammerDKSS06}. Although the PA introduces the concept of ``maximum margin" for classification, it fails to control the direction and scale of parameter updates during online learning phase. In order to address this issue, recent years have witnessed some second-order online learning algorithms~\cite{DBLP:conf/icml/DredzeCP08,DBLP:conf/nips/CrammerKD09,DBLP:conf/nips/OrabonaC10,DBLP:conf/icml/HoiWZ12}, which apply parameter confidence information to improve online learning performance. Further, in order to solve the cost-sensitive classification tasks on-the-fly, online learning researchers have also proposed a few novel online learning algorithms to directly optimize some more meaningful cost-sensitive metrics~\cite{DBLP:journals/tkde/WangZH14,DBLP:conf/kdd/ZhaoH13,DBLP:conf/sdm/HoiZ13}.

\textbf{AUC Maximization.} AUC (Area Under ROC curve) is an important performance measure that has been widely used in imbalanced data distribution classification. The ROC curve explains the rate of the true positive against the false positive at various range of threshold. Thus, AUC represents the probability that a classifier will rank a randomly chosen positive instance higher than a randomly chosen negative one. Recently, many algorithms have been developed to optimize AUC directly~\cite{DBLP:conf/nips/CortesM03,DBLP:conf/pkdd/CaldersJ07,DBLP:conf/icml/Joachims05,DBLP:conf/icml/ZhaoHJY11,DBLP:conf/icml/GaoJZZ13}.
In~\cite{DBLP:conf/icml/Joachims05}, the author firstly presented a general framework for optimizing multivariate nonlinear performance measures such as the AUC, F1, etc. in a batch mode. Online learning algorithms for AUC maximization involving large-scale applications have also been studied. Among the online AUC maximization approaches, two core online AUC optimization frameworks have been proposed very recently. The first framework is based on the idea of buffer sampling~\cite{DBLP:conf/icml/ZhaoHJY11,DBLP:conf/icml/KarS0K13}, which employed a fixed-size buffer to represent the observed data for calculating the pairwise loss functions. A representative study is available in~\cite{DBLP:conf/icml/ZhaoHJY11}, which leveraged the reservoir sampling technique to represent the observed data instances by a fixed-size buffer where notable theoretical and empirical results have been reported. Then,~\cite{DBLP:conf/icml/KarS0K13} studied the improved generalization capability of online learning algorithms for pairwise loss functions with the framework of buffer sampling. The main contribution of their work is the introduction of the stream subsampling with replacement as the buffer update strategy. The other framework which takes a different perspective was presented by~\cite{DBLP:conf/icml/GaoJZZ13}. They extended the previous online AUC maximization framework with a regression-based one-pass learning mode, and achieved solid regret bounds by considering square loss for the AUC optimization task due to its theoretical consistency with AUC.

\textbf{Sparse Online Learning.} The high dimensionality and high sparsity are two important issues for large-scale machine learning tasks. Many previous efforts have been devoted to tackling these issues in the batch setting, but they usually suffer from poor scalability when dealing with big data. Recent years have witnessed extensive research studies on sparse online learning~\cite{DBLP:journals/jmlr/LangfordLZ09,DBLP:journals/jmlr/DuchiS09,DBLP:journals/jmlr/Xiao10,DBLP:journals/jmlr/Shalev-ShwartzT11}, which aim to learn sparse classifiers by limiting the number of active features. There are two core categories of methods for sparse online learning. The representative work of the first type follows the general framework of subgradient descent with truncation. Taking the FOBOS algorithm~\cite{DBLP:journals/jmlr/DuchiS09} as an example, which is based on the Forward-Backward Splitting method to solve the sparse online learning problem by alternating between two phases: (i) an unconstraint stochastic subgradient descent step with respect to the loss function, and (ii) an instantaneous optimization for a tradeoff between keeping close proximity to the result of the first step and minimizing $\ell_1$ regularization term. Following this strategy, ~\cite{DBLP:journals/jmlr/LangfordLZ09} argues that the truncation at each step is too aggressive and thus proposes the Truncated Gradient (TG) method, which alleviates the updates by truncating the coefficients at every K steps when they are lower than a predefined threshold. The second category of methods are mainly motivated by the dual averaging method~\cite{DBLP:journals/mp/Nesterov09}. The most popular method in this category is the Regularized Dual Averaging (RDA)~\cite{DBLP:journals/jmlr/Xiao10}, which solves the optimization problem by using the running average of all past subgradients of the loss functions and the whole regularization term instead of the subgradient. In this manner, the RDA method has been shown to exploit the regularization structure more easily in the online phase and obtain the desired regularization effects more efficiently.

Despite the extensive works in these different fields of machine learning, to the best of our knowledge, our current work represents the first effort to explore adaptive gradient optimization and second order learning techniques for online AUC maximization in both regular and sparse online learning settings.

\section{Adaptive Subgradient Methods for OAM}

\subsection{Problem Setting}

We aim to learn a linear classification model that maximizes AUC for a binary classification problem. Without loss of generality, we assume positive class to be less than negative class. Denote $(\x_t, y_t)$ as the training instance received at the $t$-th trial, where $\x_t \in \mathbb{R}^d$ and $y_t \in \{-1, +1\}$, and $\mathbf{w}_t\in \mathbb{R}^d$ is the weight vector learned so far.

Given this setting, let us define the AUC measurement~\cite{Fawcett:2006:IRA:1159473.1159475} for binary classification task. Given a dataset $\D = \{(\mathbf{x}_i,y_i)\in \mathbb{R}^d\times \{-1,+1\}|\ i \in [n]\}$, where $[n]=\{1,2,\ldots,n\}$, we divide it into two sets naturally: the set of positive instances $\D_+ = \{(\mathbf{x}^+_i,+1)|\ i \in [n_+]\}$ and the set of negative instances $\D_- = \{(\mathbf{x}^-_j,-1)|\ j\in[n_-]\}$, where $n_{+}$ and $n_{-}$ are the numbers of positive and negative instances, respectively. For a linear classifier $\mathbf{w}\in\mathbb{R}^d$, its AUC measurement on $\D$ is defined as follows:
%\begin{small}
\begin{eqnarray*}
	\mbox{AUC}(\mathbf{w})=\frac{\sum^{n_+}_{i=1}\sum^{n_-}_{j=1}\mathbb{I}_{(\mathbf{w}\cdot \mathbf{x}^+_i>\mathbf{w}\cdot \mathbf{x}^-_j)}+\frac{1}{2}\mathbb{I}_{(\mathbf{w}\cdot \mathbf{x}^+_i=\mathbf{w}\cdot \mathbf{x}^-_j)} }{n_+ n_-},
\end{eqnarray*}
%\end{small}
where $\mathbb{I}_\pi$ is the indicator function that outputs a $'1'$ if the prediction $\pi$ holds and $'0'$ otherwise. We replace the indicator function with the following convex surrogate, i.e., the square loss from~\cite{DBLP:conf/icml/GaoJZZ13} due to its consistency with AUC~\cite{gao2012consistency}
\begin{eqnarray*}
	\ell(\mathbf{w}, \mathbf{x}^+_i - \mathbf{x}^-_j)=(1-\mathbf{w}\cdot(\mathbf{x}^+_i-\mathbf{x}^-_j))^2,
\end{eqnarray*}
and find the optimal classifier by minimizing the following objective function
%\begin{small}
\begin{eqnarray}
\L(\w)=\frac{\lambda}{2}\|\w\|_2^2 + \frac{\sum\limits^{n_+}_{i=1}\sum\limits^{n_-}_{j=1}\ell(\mathbf{w}, \mathbf{x}^+_i - \mathbf{x}^-_j)}{2 n_+ n_-} . \label{eqn:jin-1}
\end{eqnarray}
%\end{small}
where $\frac{\lambda}{2}\|\w\|_2^2$ is introduced to regularize the complexity of the linear classifier. Note, the optimal $\w_*$ satisfies $\|\w_*\|_2\le 1/\sqrt{\lambda}$ according to the strong duality theorem.

\subsection{Adaptive Online AUC Maximization}
Here, we shall introduce the proposed Adaptive Online AUC Maximization (AdaOAM) algorithm. Following the similar approach in \cite{DBLP:conf/icml/GaoJZZ13}, we modify the loss function $\L(\w)$ in~\eqref{eqn:jin-1} as a sum of losses for individual training instance $\sum\limits^{T}_{t=1}\L_t(\w)$ where
%\begin{small}
\begin{eqnarray}
\L_t(\w) = \frac{\lambda}{2}\|\w\|_2^2 + \frac{\sum\limits^{t-1}_{i=1}\mathbb{I}_[y_i\neq y_t](1-y_t(\mathbf{x}_t-\mathbf{x}_i)^\top\mathbf{w})^2}{2|i\in [t-1]:y_iy_t=-1|}, \label{eqn:obj1}
\end{eqnarray}
%\end{small}
for i.i.d. sequence $\S_t=\{(\x_i,y_i)|i\in[t]\}$, and it is an unbiased estimation to $\L(\w)$. $X_t^+$ and $X_t^-$ are denoted as the sets of positive and negative instances  of $\S_t$ respectively, and $T_t^+$ and $T_t^-$ are their respective cardinalities. Besides, $\L_t(\w)$ is set as $0$ for $T_t^+T_t^-=0$. If $y_t=1$, the gradient of $\L_t$ is
\begin{eqnarray*}
\begin{split}	
	\nabla\L_t(\w)=& \lambda\w+ \x_t\x_t^\top\w -\x_t  \\
	&+\frac{ \sum\limits_{i:y_i=-1}\x_i\hspace{-0.05in}+(\x_i\x_i^\top \hspace{-0.05in}-\hspace{-0.05in} \x_i\x_t^\top \hspace{-0.05in}-\hspace{-0.05in} \x_t\x_i^\top\hspace{-0.03in})\w}{T_t^-}.
\end{split}	
\end{eqnarray*}
If using $\c_t^-=\frac{1}{T_t^-}\sum\limits_{i:y_i=-1}\x_i$ and $S_t^-=\frac{1}{T_t^-}\sum\limits_{i:y_i=-1}(\x_i\x_i^\top-\c_t^-[\c_t^-]^\top)$ to refer to the mean and covariance matrix of negative class, respectively, the gradient of $\L_t$ can be simplified as
%\begin{small}
\begin{eqnarray}
\nabla\L_t(\w)=\hspace{-0.05in} \lambda\w \hspace{-0.03in}- \x_t \hspace{-0.03in}+ \c_t^- \hspace{-0.05in}+ (\x_t-\c_t^-\hspace{-0.02in})(\x_t-\c_t^-\hspace{-0.02in})^\top\w \hspace{-0.03in}+ S_t^-\w . \label{eqn:gra+1}
\end{eqnarray}
%\end{small}
\noindent Similarly, if $y_t=-1$,
%\begin{small}
\begin{eqnarray}
\nabla\L_t(\w)=\hspace{-0.05in} \lambda\w \hspace{-0.03in}+ \x_t \hspace{-0.03in}- \c_t^+ \hspace{-0.05in}+ (\x_t-\c_t^+\hspace{-0.02in})(\x_t-\c_t^+\hspace{-0.02in})^\top\w \hspace{-0.03in}+ S_t^+\w , \label{eqn:gra-1}
\end{eqnarray}
%\end{small}
\noindent where $\c_t^+=\frac{1}{T_t^+}\sum\limits_{i:y_i=1}\x_i$ and $S_t^+=\frac{1}{T_t^+}\sum\limits_{i:y_i=1}(\x_i\x_i^\top - \c_t^+[\c_t^+]^\top)$ are the mean and covariance matrix of positive class, respectively.

Upon obtaining gradient $\g_t=\nabla\L_t(\w_t)$, a simple solution is to move the weight $\w_t$ in the opposite direction of $\g_t$, while keeping $\|\w_{t+1}\|\leq 1/\sqrt{\lambda}$ via the projected gradient update~\cite{DBLP:conf/icml/Zinkevich03}
\begin{eqnarray*}
	\w_{t+1}=\Pi_{\frac{1}{\sqrt{\lambda}}}(\w_t-\eta \g_t)=\mathop{\arg\min}_{\|\w\| \leq \frac{1}{\sqrt{\lambda}}}\|\w-(\w_t-\eta \g_t)\|_2^2,
\end{eqnarray*}
due to $\|\w_{*}\|\leq 1/\sqrt{\lambda}$.

However, the above scheme is clearly insufficient, since it simply assigns different features with the same learning rate. In order to perform feature-wise gradient updating, we propose a second-order gradient optimization method, i.e., Adaptive Gradient Updating strategy, as inspired by~\cite{DBLP:journals/jmlr/DuchiHS11}. Specifically, we denote $g_{1:t}=[\g_1 ... \g_t]$ as the matrix obtained by concatenating the gradient sequences. The $i$-th row of this matrix is $\g_{1:t,i}$, which is also a concatenation of the $i$-th component of each gradient. In addition, we define the outer product matrix $G_t=\sum_{\tau=1}^t \g_{\tau}\g_{\tau}^\top$. Using these notations, the generalization of the standard adaptive gradient descent leads to the following weight update
\begin{eqnarray*}
	\w_{t+1}=\Pi_{\frac{1}{\sqrt{\lambda}}}^{G_t^{1/2}}(\w_t-\eta G_t^{-1/2} \g_t),
\end{eqnarray*}
where $\Pi_{\frac{1}{\sqrt{\lambda}}}^A(\u)=\mathop{\arg\min}_{\|\w\| \leq \frac{1}{\sqrt{\lambda}}}\|\w-\u\|_A=\mathop{\arg\min}_{\|\w\| \leq \frac{1}{\sqrt{\lambda}}}\langle\w-\u,A(\w-\u) \rangle$, which is the Mahalanobis norm to denote the projection of a point $\u$ onto $\{\w|\|\w\| \le \frac{1}{\sqrt{\lambda}}\}$.

However, an obvious drawback of the above update lies in the significantly large amount of computational efforts needed to handle high-dimensional data tasks since it requires the calculations of the root and inverse root of the outer product matrix $G_t$. In order to make the algorithm more efficient, we use the diagonal proxy of $G_t$ and thus the update becomes
\begin{eqnarray}\label{eqn:diag-update-rule}
\w_{t+1}=\Pi_{\frac{1}{\sqrt{\lambda}}}^{diag(G_t)^{1/2}}(\w_t-\eta diag(G_t)^{-1/2} \g_t).
\end{eqnarray}
In this way, both the root and inverse root of $diag(G_t)$ can be computed in linear time. Furthermore, as we discuss later, when the gradient vectors are sparse, the update above can be conducted more efficiently in time proportional to the support of the gradient.

Another issue with the updating rule~\eqref{eqn:diag-update-rule} to be concern is that the $diag(G_t)$ may not be invertible in all coordinates. To address this issue, we replace it with $H_t=\delta I + diag(G_t)^{1/2}$, where $\delta>0$ is a smooth parameter. The parameter $\delta$ is introduced to make the diagonal matrix invertible and the algorithm robust, which is usually set as a very small value so that it has little influence on the learning results. Given $H_t$, the update of the feature-wise adaptive update can be computed as:
\begin{eqnarray}\label{eqn:update-H}
\w_{t+1}=\Pi_{\frac{1}{\sqrt{\lambda}}}^{H_t}(\w_t-\eta H_t^{-1} \g_t).
\end{eqnarray}

The intuition of this update rule~\eqref{eqn:update-H} is very natural, which considers the rare occurring features as more informative and discriminative than those frequently occurring features. Thus, these informative rare occurring features should be updated with higher learning rates by incorporating the geometrical property of the data observed in earlier stages. Besides, by using the previously observed gradients, the update process can mitigate the effects of noise and speed up the convergence rate intuitively.

So far, we have reached the key framework of the basic update rule for model learning except the details on gradient calculation. From the gradient derivation equations~\eqref{eqn:gra+1} and~\eqref{eqn:gra-1}, we need to maintain and update the mean vectors and covariance matrices of the incoming instance sequences observed. The mean vectors are easy to be computed and stored here, while the covariance matrices are a bit difficult to be updated due to the online setting. Therefore, we provide a simplified update scheme for covariance matrix computation to address this issue. For an incoming instance sequence $\x_1, \x_2, \dots, \x_t$, the covariance matrix $S_t$ is given by

\begin{small}
\begin{eqnarray*}
\begin{split}	
	S_t  =& \sum\limits_{i=1}^t \frac{\x_i\x_i^\top-\c_t[\c_t]^\top}{t} \\
	 =& \sum\limits_{i=1}^t \frac{\x_i\x_i^\top}{t}-\c_t[\c_t]^\top  \\
	 =& \frac{t-1}{t}\sum\limits_{i=1}^{t-1}\frac{\x_i\x_i^\top}{t-1}+\frac{\x_t\x_t^\top}{t}-\c_t[\c_t]^\top \\ 
	 =& \frac{t-1}{t}(\sum\limits_{i=1}^{t-1}\frac{\x_i\x_i^\top}{t-1}-\c_{t-1}[\c_{t-1}]^\top+\c_{t-1}[\c_{t-1}]^\top) +\frac{\x_t\x_t^\top}{t} \\&-\c_t[\c_t]^\top  \\
	 =& \frac{t-1}{t}(S_{t-1}+\c_{t-1}[\c_{t-1}]^\top)+\frac{\x_t\x_t^\top}{t}-\c_t[\c_t]^\top  \\
	 =& S_{t-1} -\frac{1}{t}(S_{t-1}+\c_{t-1}[\c_{t-1}]^\top-\x_t\x_t^\top)+\c_{t-1}[\c_{t-1}]^\top\\
	 &-\c_t[\c_t]^\top .
\end{split}	
\end{eqnarray*}
\end{small}

Then, in our gradient update, if setting $\Gamma_t^+=S_t^+$ and $\Gamma_t^-=S_t^-$, the covariance matrices are updated as follows:
\begin{small}
	\begin{eqnarray*}
		&&\hspace{-0.25in}\Gamma_t^+\hspace{-0.05in} = \Gamma_{t-1}^+ \hspace{-0.05in}+\hspace{-0.02in} \c_{t-1}^+[\c_{t-1}^+]^\top \hspace{-0.05in}-\hspace{-0.02in} \c_{t}^+[\c_{t}^+]^\top \hspace{-0.05in}+\hspace{-0.02in} \frac{\x_t\x_t^\top\hspace{-0.05in}-\Gamma_{t-1}^+ \hspace{-0.05in}- \c_{t-1}^+[\c_{t-1}^+]^\top}{T_t^+}, \\
		&&\hspace{-0.25in}\Gamma_t^- \hspace{-0.05in}= \Gamma_{t-1}^- \hspace{-0.05in}+ \c_{t-1}^-[\c_{t-1}^-]^\top \hspace{-0.05in}- \c_{t}^-[\c_{t}^-]^\top \hspace{-0.05in}+ \frac{\x_t\x_t^\top\hspace{-0.05in}-\hspace{-0.03in}\Gamma_{t-1}^- \hspace{-0.05in}- \hspace{-0.03in}\c_{t-1}^-[\c_{t-1}^-]^\top}{T_t^-}.
	\end{eqnarray*}
\end{small}

It can be observed that the above updates fit the online setting well. Finally, Algorithm~\ref{alg:AdaOAM} summarizes the proposed AdaOAM method.

\begin{algorithm}[htbp]
	\caption{The Adaptive OAM Algorithm (AdaOAM)} \label{alg:AdaOAM}
	\begin{algorithmic}
		\STATE {\bf Input}: The regularization parameter $\lambda$, the learning rate $\{\eta_t\}_{t=1}^T$, the smooth parameter $\delta\ge 0$.
		\STATE {\bf Output}: Updated classifier $\w_t$.
		\STATE {\bf Variables}: $s\in \mathbb{R}^d$, $H\in \mathbb{R}^{d\times d}$, $g_{1:t,i}\in \mathbb{R}^t$ for $i\in \{1,\dots,d\}$.
		\STATE {\bf Initialize} $\mathbf{w}_0 = \mathbf{0}, \c^+_{0}= \c^-_{0}=\mathbf{0}, \Gamma^+_{0} = \Gamma^-_{0} = [0]_{d\times d}, T^+_{0} = T^-_{0} = 0, g_{1:0} = [\,]$.
		\FOR{$t=1,2,\ldots,T$}
		\STATE Receive an incoming instance $(\x_t,y_t)$;
		\IF{$y_t = +1$}
		\STATE $T_t^+ = T_{t-1}^+ + 1$, $T_t^- = T_{t-1}^-$;
		\STATE $\c_t^{+} = \c_{t-1}^+ + \frac{1}{T_t^+}(\x_t-\c_{t-1}^+)$ and $\c_t^{-} = \c_{t-1}^-$;
		\STATE Update $\Gamma^+_{t}$ and $\Gamma^-_{t} = \Gamma^-_{t-1}$;
		\STATE Receive gradient $\g_t=\nabla\L_t(\w)$;
		\STATE Update $g_{1:t}=[g_{1:t-1} \, \g_t]$, $s_{t,i}=\|g_{1:t,i}\|_2$;
		\STATE $H_t = \delta I + diag(s_t)$;
		\STATE $\hat{\g}_t=H_t^{-1} \g_t$;
		\ELSE
		\STATE $T_t^- = T_{t-1}^- + 1$, $T_t^+ = T_{t-1}^+$;
		\STATE $\c_t^{-} = \c_{t-1}^- + \frac{1}{T_t^-}(\x_t-\c_{t-1}^-)$ and $\c_t^{+} = \c_{t-1}^+$;
		\STATE Update $\Gamma^-_{t}$ and $\Gamma^+_{t} = \Gamma^+_{t-1}$;
		\STATE Receive gradient $\g_t=\nabla\L_t(\w)$;
		\STATE Update $g_{1:t}=[g_{1:t-1} \, \g_t]$, $s_{t,i}=\|g_{1:t,i}\|_2$;
		\STATE $H_t = \delta I + diag(s_t)$;
		\STATE $\hat{\g}_t=H_t^{-1} \g_t$;
		\ENDIF
		\STATE $\w_t = \Pi_{\frac{1}{\sqrt{\lambda}}}^{H_t}(\w_{t-1} - \eta_t\hat{\g}_t)$;
		\ENDFOR
	\end{algorithmic}
\end{algorithm}

\subsection{Fast AdaOAM Algorithm for High-dimensional Sparse Data}

A characteristic of the proposed AdaOAM algorithm described above is that it exploits the full features for weight learning, which may not be suitable or scalable for high-dimensional sparse data tasks. For example, in spam email detection tasks, the length of the vocabulary list can reach the million scale. Although the number of the features is large, many feature inputs are zero and do not provide any information to the detection task. The research work in~\cite{DBLP:journals/jsw/YounM07} has shown that the classification performance saturates with dozens of features out of tens of thousands of features.

Taking the cue, in order to improve the efficiency and scalability of the AdaOAM algorithm on working with high-dimensional sparse data, we propose the Sparse AdaOAM algorithm (SAdaOAM) which learns a sparse linear classifier that contains a limited size of active features. In particular, SAdaOAM addresses the issue of sparsity in the learned model and maintains the efficacy of the original AdaOAM at the same time. To summarize, SAdaOAM has two benefits over the original AdaOAM algorithm: simple covariance matrix update and sparse model update. Next, we introduce these properties separately.

First, we employ a simpler covariance matrix update rule in the case of handling high-dimensional sparse data when compared to AdaOAM. The motivating factor behind a different update scheme here is because using the original covariance update rule of the AdaOAM method on high-dimensional data would lead to extreme high computational and storage costs, i.e. several matrix operations among multiple variables in the $\Gamma$ update formulations would be necessary. Therefore, we fall back to the standard definition of the covariance matrix and consider a simpler method for updates. Since the standard definition of the covariance matrix is $S_t  = \sum\limits_{i=1}^t \frac{\x_i\x_i^\top}{t}-\c_t[\c_t]^\top$, we just need to maintain the mean vector and the outer product of the instance at each iteration for the covariance update. In this case, we denote $Z_t^+=\sum\limits_{i=1}^t \x_i^+[\x_i^+]^\top$ and $Z_t^-=\sum\limits_{i=1}^t \x_i^-[\x_i^-]^\top$. Then, the covariance matrices $S_t^+$ and $S_t^-$ can be formulated as
\begin{eqnarray*}
	S_t^+=Z_t^+/T_t^+ - \c_t^+[\c_t^+]^\top  \quad \text{and}  \quad
	S_t^-=Z_t^-/T_t^- - \c_t^-[\c_t^-]^\top.
\end{eqnarray*}
At each iteration, one only needs to update $Z_t^{y_t}$ with $Z_t^{y_t}=Z_{t-1}^{y_t}+\x_t[\x_t]^\top$ and the mean vectors of the positive and negative instances, respectively, in the covariance matrices $S_t^+$ and $S_t^-$. With the above update scheme, a lower computational and storage costs is attained since most of the elements in the covariance matrices are zero on high-dimensional sparse data.

After presenting the efficient scheme of updating the covariance matrices for high-dimensional sparse data, we proceed next to present the method for addressing the sparsity in the learned model and second-order adaptive gradient updates simultaneously. Here we consider to impose the soft-constraint $\ell_1$ norm regularization $\varphi(\w) = \theta \|\w\|_1$ to the objective function~\eqref{eqn:obj1}. So, the new objective function is
\begin{eqnarray}
\L_t(\w) = \frac{\lambda}{2}\|\w\|_2^2 + \frac{\sum\limits^{t-1}_{i=1}\mathbb{I}_[y_i\neq y_t](1-y_t(\mathbf{x}_t-\mathbf{x}_i)^\top\mathbf{w})^2}{2|i\in [t-1]:y_iy_t=-1|} + \theta \|\w\|_1 , \label{eqn:obj2}
\end{eqnarray}
In order to optimize this objective function, we apply the composite mirror descent method~\cite{paultseng2008} that is able to achieve a trade-off between the immediate adaptive gradient term $\g_t$ and the regularizer $\varphi(\w)$. We denote the $i$-th diagonal element of the matrix $H_t$ as $H_{t,ii}=\delta I + \|g_{1:t,i}\|_2$. Then, we give the derivation for the composite mirror descent gradient updates with $\ell_1$ regularization.

Following the framework of the composite mirror descent update in~\cite{DBLP:journals/jmlr/DuchiHS11}, the update needed to solve is
\begin{eqnarray}
\w_{t+1}=\mathop{\arg\min}\limits_{\|\w\|\le\frac{1}{\sqrt{\lambda}}}\{\eta\langle \g_t,\w\rangle + \eta\varphi(\w) + B_{\psi_t}(\w,\w_t)\}, \label{eqn:4}
\end{eqnarray}
where $\g_t=\nabla\L_t(\w)$ and $B_{\psi_t}(\w,\w_t)$ is the Bregman divergence associated with $\psi_t(\g_t)=\langle \g_t,H_t \g_t\rangle$ (see details in the proof of Theorem~\ref{thm:1}). After the expansion, this update amounts to
\begin{eqnarray}
\mathop{\min}\limits_{\w} \eta\langle \g_t,\w\rangle + \eta\varphi(\w) + \frac{1}{2}\langle \w-\w_t, H_t(\w-\w_t)\rangle.
\end{eqnarray}
For easier derivation, we rearrange and rewrite the above objective function as
\begin{eqnarray*}
	\mathop{\min}\limits_{\w} \langle \eta\g_t - H_t\w_t,\w\rangle + \frac{1}{2}\langle \w, H_t\w\rangle + \frac{1}{2}\langle \w_t, H_t\w_t\rangle + \eta\theta\|\w\|_1.
\end{eqnarray*}
Let $\hat{\w}$ denote the optimal solution of the above optimization problem. Standard subgradient calculus indicates that when $|\w_{t,i}-\frac{\eta}{H_{t,ii}}\g_{t,i}| \leq \frac{\eta\theta}{H_{t,ii}}$, the solution is $\hat{\w}_i=0$. Similarly, when $\w_{t,i}-\frac{\eta}{H_{t,ii}}\g_{t,i} < -\frac{\eta\theta}{H_{t,ii}}$, then $\hat{\w}_i<0$, the objective is differentiable, and the solution is achieved by setting the gradient to zero:
\begin{eqnarray*}
	\eta \g_{t,i}-H_{t,ii}\w_{t,i}-H_{t,ii}\hat{\w}_i - \eta \theta = 0,
\end{eqnarray*}
so that
\begin{eqnarray*}
	\hat{\w}_i = \frac{\eta}{H_{t,ii}}\g_{t,i}-\w_{t,i}-\frac{\eta\theta}{H_{t,ii}}.
\end{eqnarray*}
Similarly, when $\w_{t,i}-\frac{\eta}{H_{t,ii}}\g_{t,i} > \frac{\eta\theta}{H_{t,ii}}$, then $\hat{\w}_i>0$, and the solution is
\begin{eqnarray*}
	\hat{\w}_i = \w_{t,i}-\frac{\eta}{H_{t,ii}}\g_{t,i}-\frac{\eta\theta}{H_{t,ii}}.
\end{eqnarray*}
Combining these three cases, we obtain the coordinate-wise update results for $\w_{t+1,i}$ :
\begin{eqnarray*}
	\w_{t+1,i}=sign(\w_{t,i}-\frac{\eta}{H_{t,ii}}\g_{t,i})[|\w_{t,i}-\frac{\eta}{H_{t,ii}}\g_{t,i}|-\frac{\eta\theta}{H_{t,ii}}]_+.
\end{eqnarray*}

The complete sparse online AUC maximization approach using the adaptive gradient updating algorithm (SAdaOAM) is summarized in Algorithm~\ref{alg:AdaOAMs}.

\begin{algorithm}[htbp]
	\caption{The Sparse AdaOAM Algorithm (SAdaOAM)} \label{alg:AdaOAMs}
	\begin{algorithmic}
		\STATE {\bf Input}: The regularization parameters $\lambda$ and $\theta$, the learning rate $\{\eta_t\}_{t=1}^T$, the smooth parameter $\delta\ge 0$.
		\STATE {\bf Output}: Updated classifier $\w_{t+1}$.
		\STATE {\bf Variables}: $s\in \mathbb{R}^d$, $H\in \mathbb{R}^{d\times d}$, $g_{1:t,i}\in \mathbb{R}^t$ for $i\in \{1,\dots,d\}$.
		\STATE {\bf Initialize} $\mathbf{w}_0 = \mathbf{0}, \c^+_{0}= \c^-_{0}=\mathbf{0}, Z^+_{0} = Z^-_{0} = [0]_{d\times d}, T^+_{0} = T^-_{0} = 0, g_{1:0} = [\,]$.  \vspace{-0.2in}
		\FOR{$t=1,2,\ldots,T$}
		\STATE Receive an incoming instance $(\x_t,y_t)$;
		\IF{$y_t = +1$}
		\STATE $T_t^+ = T_{t-1}^+ + 1$, $T_t^- = T_{t-1}^-$;
		\STATE $\c_t^{+} = \c_{t-1}^+ + \frac{1}{T_t^+}(\x_t-\c_{t-1}^+)$ and $\c_t^{-} = \c_{t-1}^-$;
		\STATE Update $Z^+_{t}=Z^+_{t-1}+\x_t[\x_t]^\top$ and $Z^-_{t} = Z^-_{t-1}$;
		\STATE Receive gradient $\g_t=\nabla\L_t(\w)$;
		\STATE Update $g_{1:t}=[g_{1:t-1} \, \g_t]$, $s_{t,i}=\|g_{1:t,i}\|_2$;
		\STATE $H_{t,ii}=\delta I + \|g_{1:t,i}\|_2$;
		\STATE $\w_{t+1,i}=sign(\w_{t,i}-\frac{\eta}{H_{t,ii}}\g_{t,i})[|\w_{t,i}-\frac{\eta}{H_{t,ii}}\g_{t,i}|-\frac{\eta\theta}{H_{t,ii}}]_+$;
		\ELSE
		\STATE $T_t^- = T_{t-1}^- + 1$, $T_t^+ = T_{t-1}^+$;
		\STATE $\c_t^{-} = \c_{t-1}^- + \frac{1}{T_t^-}(\x_t-\c_{t-1}^-)$ and $\c_t^{+} = \c_{t-1}^+$;
		\STATE Update $Z^-_{t}=Z^-_{t-1}+\x_t[\x_t]^\top$ and $Z^+_{t} = Z^-_{t-1}$;
		\STATE Receive gradient $\g_t=\nabla\L_t(\w)$;
		\STATE Update $g_{1:t}=[g_{1:t-1} \, \g_t]$, $s_{t,i}=\|g_{1:t,i}\|_2$;
		\STATE $H_{t,ii}=\delta I + \|g_{1:t,i}\|_2$;
		\STATE $\w_{t+1,i}=sign(\w_{t,i}-\frac{\eta}{H_{t,ii}}\g_{t,i})[|\w_{t,i}-\frac{\eta}{H_{t,ii}}\g_{t,i}|-\frac{\eta\theta}{H_{t,ii}}]_+$;
		\ENDIF
		\ENDFOR
	\end{algorithmic}
\end{algorithm}

From the Algorithm~\ref{alg:AdaOAMs}, it is observed that we perform ``lazy" computation when the gradient vectors are sparse~\cite{DBLP:journals/jmlr/DuchiHS11}. Suppose that, from iteration step $t_0$ to $t$, the $i$-th component of the gradient is ``0". Then, we can evaluate the updates on demand since $H_{t,ii}$ remains intact. Therefore, at iteration step $t$ when $\w_{t,i}$ is needed, the update will be
\begin{eqnarray*}
	\w_{t,i}=sign(\w_{t_0},i)[|\w_{t_0},i|-\frac{\eta\theta}{H_{t_0,ii}}(t-t_0)]_+ ,
\end{eqnarray*}
where $[t]_+$ means $\max(0,t)$. Obviously, this type of "lazy" updates enjoys high efficiency.

\section{Theoretical Analysis}

In this section, we provide the regret bounds for the proposed set of AdaOAM algorithms for handling both regular and high-dimensional sparse data, respectively.

\subsection{Regret Bounds with Regular Data}
Firstly, we introduce two lemmas as follows, which will be used to facilitate our subsequent analyses.
\begin{lemma}\label{lemma:1}
	Let $\g_t$, $g_{1:t}$ and $s_t$ be defined same in the Algorithm~\ref{alg:AdaOAM}. Then
	\begin{eqnarray*}\vspace{-0.1in}
		\sum\limits_{t=1}^T\langle \g_t,diag(s_t)^{-1} \g_t\rangle \leq 2\sum\limits_{i=1}^d \|g_{1:T,i}\|_2.
	\end{eqnarray*}
\end{lemma}

\begin{lemma}\label{lemma:2}
	Let the sequence $\{\w_t$, $g_{1:t}\} \subset \mathbb{R}^d $ be generated by the composite mirror descent update in Equation~\eqref{eqn:4} and assume that $sup_{\w,\u\in \chi}\|\w-\u\|_{\infty} \leq D_{\infty}$.
	Using learning rate $\eta=D_{\infty}/\sqrt{2}$, for any optimal $\w_*$, the following bound holds
	\begin{eqnarray*}\vspace{-0.1in}
		\sum\limits_{t=1}^T[ \L_t(\w_t)-\L_t(\w_*)] \leq \sqrt{2}D_{\infty}\sum\limits_{i=1}^d \|g_{1:T,i}\|_2 .
	\end{eqnarray*}
\end{lemma}

These two lemmas are actually the Lemma 4 and Corollary 1 in the paper~\cite{DBLP:journals/jmlr/DuchiHS11}.

Using these two lemmas, we can derive the following theorem for the proposed AdaOAM algorithm.
\begin{thm}\label{thm:1}
	Assume $\|\w_t\| \le 1/\sqrt{\lambda},\quad (\forall t\in[T])$ and the diameter of $\chi=\{\w|\|\w\|\le\frac{1}{\sqrt{\lambda}}\}$ is bounded via $sup_{\w,\u\in \chi}\|\w-\u\|_{\infty} \leq D_{\infty}$, we have
	%\begin{small}
	\begin{eqnarray*}\vspace{-0.1in}
		\sum\limits_{t=1}^T[ \L_t(\w_t)-\L_t(\w_*)] \leq 2D_{\infty}\sum\limits_{i=1}^d \sqrt{\sum\limits_{t=1}^T[(\lambda\w_{t,i})^2+ C(r_{t,i})^2]},\vspace{-0.1in}
	\end{eqnarray*}
	%\end{small}
	where $C\le(1+\frac{2}{\sqrt{\lambda}})^2$, and $r_{t,i}=\max_{j<t}|\x_{j,i}-\x_{t,i}|$.\vspace{-0.1in}
\end{thm}

\begin{proof}
	We first define $\w_*$ as $\w_*=\mathop{\arg\min}\limits_{\w}\sum_t \L_t(\w).$
	
	Based on the regularizer $\frac{\lambda}{2}\|\w\|^2$, it is easy to obtain $\|\w_*\|^2 \leq 1/{\lambda}$ due to the strong convexity property, and it is also reasonable to restrict $\w_t$ with $\|\w_t\|^2 \leq 1/{\lambda}$. Denote the projection of a point $\w$ onto $\|\u\|_2 \le\frac{1}{\sqrt{\lambda}}$ according to norm $\|\cdot\|_{H_t}$ by $\Pi^{H_t}_{\frac{1}{\sqrt{\lambda}}}(\w)=\arg\min_{\|\u\|\le\frac{1}{\sqrt{\lambda}}}\|\u-\w\|_{H_t}$,  the AdaOAM actually employs the following update rule:
	%\begin{small}
	\begin{eqnarray}
	\w_{t+1}=\Pi^{H_t}_{\frac{1}{\sqrt{\lambda}}}(\w_t-\eta H_t^{-1}\g_t),  \label{eqn:update}
	\end{eqnarray}
	where $H_t=\delta I + diag(s_t)$ and  $\delta \geq 0$.
	
	If we denote $\psi_t(\g_t)=\langle \g_t,H_t \g_t\rangle$, and the dual norm of $\|\cdot\|_{\psi_t}$ by $\|\cdot\|_{\psi_t^*}$, in which case $\|\g_t\|_{\psi_t^*}=\|\g_t \|_{H_t^{-1}}$, then it is easy to check the update rule~\ref{eqn:update} is the same with the following composite mirror descent method:
	\begin{eqnarray}
	\w_{t+1}=\mathop{\arg\min}\limits_{\|\w\|\le\frac{1}{\sqrt{\lambda}}}\{\eta\langle \g_t,\w\rangle + \eta\varphi(\w) + B_{\psi_t}(\w,\w_t)\}, \label{eqn:4}
	\end{eqnarray}
	where the regularization function $\varphi\equiv 0$, and $B_{\psi_t}(\w,\w_t)$ is the Bregman divergence associated with a strongly convex and differentiable function $\psi_t$
	\begin{eqnarray*}
		B_{\psi_t}(\w,\w_t)=\psi_t(\w)-\psi_t(\w_t)-\langle \nabla \psi_t(\w_t),\w-\w_t\rangle .
	\end{eqnarray*}
	
	Since we have $\varphi\equiv 0$ in the case of regular data, the regret bound $R(T)=\sum\limits_{t=1}^T [\L_t(\w_t)-\L_t(\w_*)]$. Then, we follow the derivation results of~\cite{DBLP:journals/jmlr/DuchiHS11} and attain the following regret bound
	\begin{eqnarray*}
		\sum\limits_{t=1}^T [\L_t(\w_t)-\L_t(\w_*)] \leq \sqrt{2}D_{\infty}\sum\limits_{i=1}^d \|g_{1:T,i}\|_2 ,
	\end{eqnarray*}
	where $\chi=\{\w|\|\w\|\le\frac{1}{\sqrt{\lambda}}\}$ is bounded via $sup_{\w,\u\in \chi}\|\w-\u\|_{\infty} \leq D_{\infty}$. Next, we would like to analyze the features' dependency on the data of the gradient. Since
	%\begin{small}
	\begin{eqnarray*}
		&&\hspace{-0.3in}(g_{t,i})^2 \leq \Big[\lambda \w_{t,i} + \frac{\sum\limits_{j=1}^{t-1}(1-y_t\langle\x_t-\x_j,\w\rangle)y_t(\x_{j,i}-\x_{t,i})}{T_t^-}\Big]^2 \nonumber \\
		&&\hspace{-0.3in}\leq 2(\lambda\w_{t,i})^2 + 2C(\x_{j,i}-\x_{t,i})^2=2(\lambda\w_{t,i})^2+ 2C(r_{t,i})^2,
	\end{eqnarray*}
	%\end{small}
	where $C\le (1+\frac{2}{\sqrt{\lambda}})^2$ is a constant to bound the scalar of the second term in the right side of the equation, and with $r_{t,i}=\max_{j<t}{|\x_{j,i}-\x_{t,i}|}$, we have
	\begin{small}
		\begin{eqnarray*}
			\sum\limits_{i=1}^d \|g_{1:T,i}\|_2& =& \sum^d_{i=1} \sqrt{\sum\limits_{t=1}^T (g_{t,i})^2} \\
			&\leq& \sqrt{2}\sum\limits_{i=1}^d \sqrt{\sum\limits_{t=1}^T[(\lambda\w_{t,i})^2+ C(r_{t,i})^2]}.
		\end{eqnarray*}
	\end{small}

	Finally, combining the above inequalities, we arrive at
	\begin{small}
		\begin{eqnarray*}
			\sum\limits_{t=1}^T[ \L_t(\w_t)-\L_t(\w_*)] \leq 2D_{\infty}\sum\limits_{i=1}^d \sqrt{\sum\limits_{t=1}^T[(\lambda\w_{t,i})^2+ C(r_{t,i})^2]}.
		\end{eqnarray*}
	\end{small}
\end{proof}
From the proof above, we can conclude that Algorithm~\ref{alg:AdaOAM} should have a lower regret than non-adaptive algorithms due to its dependence on the geometry of the underlying data space. If the features are normalized and sparse, the gradient terms in the bound $\sum\limits_{i=1}^d \|g_{1:T,i}\|_2$ should be much smaller than $\sqrt{T}$, which leads to lower regret and faster convergence. If the feature space is relative dense, then the convergence rate will be $O(1/\sqrt{T})$ for the general case as in OPAUC and OAM methods.

\subsection{Regret Bounds with High-dimensional Sparse Data}
\begin{thm}\label{thm:2}
	Assume $\|\w_t\| \le 1/\sqrt{\lambda},\quad (\forall t\in[T])$ and the diameter of $\chi=\{\w|\|\w\|\le\frac{1}{\sqrt{\lambda}}\}$ is bounded via $sup_{\w,\u\in \chi}\|\w-\u\|_{\infty} \leq D_{\infty}$, the regret bound with respect to $\ell_1$ regularization term is
	%\begin{small}
	\begin{eqnarray*}\vspace{-0.1in}
		\sum\limits_{t=1}^T[ \L_t(\w_t)+\theta\|\w_t\|_1-\L_t(\w_*)-\theta\|\w_*\|_1] \\ \leq 2D_{\infty}\sum\limits_{i=1}^d \sqrt{\sum\limits_{t=1}^T[(\lambda\w_{t,i})^2+ C(r_{t,i})^2]},\vspace{-0.1in}
	\end{eqnarray*}
	%\end{small}
	where $\w_*=\arg\min_\w\sum^T_{t=1}[\L_t(\w)+\theta\|\w\|_1]$, $C\le(1+\frac{2}{\sqrt{\lambda}})^2$, and $r_{t,i}=\max_{j<t}|\x_{j,i}-\x_{t,i}|$.\vspace{-0.1in}
\end{thm}
\begin{proof}
	
	In the case of high-dimensional sparse adaptive online AUC maximization, the regret we plan to bound with respect to the optimal weight $\w_*$ is formulated as
	\begin{small}
		\begin{eqnarray*}
			R(T)=\sum\limits_{t=1}^T[ \L_t(\w_t)+\varphi(\w_t)-\L_t(\w_*)-\varphi(\w_*)], \vspace{-0.1in}
		\end{eqnarray*}
	\end{small}
	where $\varphi(\w)=\theta\|\w\|_1$ is the $\ell_1$ regularization term to impose sparsity to the solution. Similarly, if denote $\psi_t(\g_t)=\langle \g_t,H_t \g_t\rangle$, and the dual norm of $\|\cdot\|_{\psi_t}$ by $\|\cdot\|_{\psi_t^*}$, in which case $\|\g_t\|_{\psi_t^*}=\|\g_t \|_{H_t^{-1}}$,  it is easy to check the updating rule of SAdaOAM
	\begin{small}
		\begin{eqnarray*}
			\w_{t+1,i}=sign(\w_{t,i}-\frac{\eta}{H_{t,ii}}\g_{t,i})[|\w_{t,i}-\frac{\eta}{H_{t,ii}}\g_{t,i}|-\frac{\eta\theta}{H_{t,ii}}]_+,
		\end{eqnarray*}
	\end{small}
	is the same with the following one
	\begin{eqnarray}
	\w_{t+1}=\mathop{\arg\min}\limits_{\|\w\|\le\frac{1}{\sqrt{\lambda}}}\{\eta\langle \g_t,\w\rangle + \eta\varphi(\w) + B_{\psi_t}(\w,\w_t)\}, \label{eqn:4}
	\end{eqnarray}
	where $\varphi(\w)=\theta||\w||_1$. From~\cite{DBLP:journals/jmlr/DuchiHS11}, we have
	\begin{eqnarray*}
		R(T) \leq \frac{1}{2\eta}\max\limits_{t\leq T}\|\w_*-\w_t\|_{\infty}^2\sum\limits_{i=1}^d \|g_{1:T,i}\|_2+\eta\sum\limits_{i=1}^d \|g_{1:T,i}\|_2 .   \label{eqn:corollary6}
	\end{eqnarray*}
	Furthermore, we assume $\sup_{\w,\u\in \chi}\|\w-\u\|_{\infty} \leq D_{\infty}$ and set $\eta=D/\sqrt{2}$, the final regret bound is
	\begin{eqnarray*}
		R(T) \leq \sqrt{2}D_{\infty}\sum\limits_{i=1}^d \|g_{1:T,i}\|_2 .
	\end{eqnarray*}
	This theoretical result shows that the regret bound for sparse solution is the same as that in the case when $\varphi\equiv 0$.
\end{proof}
As discussed above, the SAdaOAM algorithm should have lower regret bound than non-adaptive algorithms do on high-dimensional sparse data, though this depends on the geometric property of the underlying feature distribution. If some features appear much more frequently than others, then $\sum\limits_{i=1}^d \|g_{1:T,i}\|_2$ indicates that we could have remarkably lower regret by using higher learning rates for infrequent features and lower learning rates for often occurring features.

\section{Experimental Results}

In this section, we evaluate the proposed set of the AdaOAM algorithms in terms of AUC performance, convergence rate, and examine their parameter sensitivity. The main framework of the experiments is based on the LIBOL, an open-source library for online learning algorithms~\footnote{\url{http://libol.stevenhoi.org/}}~\cite{DBLP:journals/jmlr/HoiWZ14}.

\subsection{Comparison Algorithms}
We conduct comprehensive empirical studies by comparing the proposed algorithms with various AUC optimization algorithms for both online and batch scenarios. Specifically, the algorithms considered in our experiments include:

\begin{itemize}
	\item \textbf{Online Uni-Exp}: An online learning algorithm which optimizes the (weighted) univariate exponential loss~\cite{DBLP:conf/icml/KotlowskiDH11};\vspace{-0.05in}
	\item \textbf{Online Uni-Log}: An online learning algorithm which optimizes the (weighted) univariate logistic loss~\cite{DBLP:conf/icml/KotlowskiDH11};\vspace{-0.05in}
	\item \textbf{OAM$_{\textbf{seq}}$}: The OAM algorithm with reservoir sampling and sequential updating method~\cite{DBLP:conf/icml/ZhaoHJY11};\vspace{-0.05in}
	\item \textbf{OAM$_{\textbf{gra}}$}: The OAM algorithm with reservoir sampling and online gradient updating method~\cite{DBLP:conf/icml/ZhaoHJY11};\vspace{-0.05in}
	\item \textbf{OPAUC}: The one-pass AUC optimization algorithm with square loss function~\cite{DBLP:conf/icml/GaoJZZ13};\vspace{-0.05in}
	\item \textbf{SVM-perf}: A batch algorithm which directly optimizes AUC~\cite{DBLP:conf/icml/Joachims05};\vspace{-0.05in}
	\item \textbf{CAPO}: A batch algorithm which trains nonlinear auxiliary classifiers first  and then adapts auxiliary classifiers for specific performance measures~\cite{DBLP:journals/pami/LiTZ13};\vspace{-0.05in}
	\item \textbf{Batch Uni-Log}: A batch algorithm which optimizes the (weighted) univariate logistic loss~\cite{DBLP:conf/icml/KotlowskiDH11};\vspace{-0.05in}
	\item \textbf{Batch Uni-Squ}: A batch algorithm which optimizes the (weighted) univariate square loss;\vspace{-0.05in}
	\item \textbf{AdaOAM}: The proposed adaptive gradient method for online AUC maximization.\vspace{-0.05in}
	\item \textbf{SAdaOAM}: The proposed sparse adaptive subgradient method for online AUC maximization.\vspace{-0.05in}
\end{itemize}

It is noted that the OAM$_{seq}$, OAM$_{gra}$, and OPAUC are the state-of-the-art methods for AUC maximization in online settings. For batch learning scenarios, CAPO and SVM-perf are both strong baselines to compare against.

\subsection{Experimental Testbed and Setup}

To examine the performance of the proposed AdaOAM in comparison to the existing state-of-the-art methods, we conduct extensive experiments on sixteen benchmark datasets by maintaining consistency to the previous studies on online AUC maximization~\cite{DBLP:conf/icml/ZhaoHJY11,DBLP:conf/icml/GaoJZZ13}. Table~\ref{table:datasets} shows the details of 16 binary-class datasets in our experiments.  All of these datasets can be downloaded from the LIBSVM~\footnote{\url{http://www.csie.ntu.edu.tw/~cjlin/libsvmtools/}} and UCI machine learning repository~\footnote{\url{http://www.ics.uci.edu/~mlearn/MLRepository.html}}. Note that several datasets (svmguide4, vehicle) are originally multi-class, which were converted to class-imbalanced binary datasets for the purpose of in our experimental studies.
\begin{table*}[htpb]
	\vspace{-0.2in}
	\renewcommand*\arraystretch{1.0}
	\begin{center}
		\caption{\small Details of benchmark machine learning datasets.}\label{table:datasets}
		\vspace{-0.1in}
		\begin{small}
			\begin{tabular}{|c|ccc|c|ccc|}        \hline
				datasets &   $\#$ inst   & $\#$ dim &$T_{-} / T_{+}$ & datasets & $\#$ inst    &$\#$ dim &$T_{-} / T_{+}$\\
				\hline\hline
				glass           &214     &9   &2.057   &vehicle         &846     &18  & 3.251\\
				heart           &270     &13  &1.250   &german          &1,000   &24  & 2.333\\
				svmguide4       &300     &10  &5.818   &svmguide3       &1,243   &22  & 3.199\\
				liver-disorders &345     &6   &1.379   &a2a             &2,265   &123 & 2.959\\
				balance         &625     &4   &11.755  &magic04         &19,020  &10  & 1.843\\
				breast          &683     &10  &1.857   &cod-rna         &59,535  &8   & 2.000\\
				australian      &690     &14  &1.247   &acoustic        &78,823  &50  & 3.316\\
				diabetes        &768     &8   &1.865   &poker           &1025,010&11  & 10.000\\
				\hline
			\end{tabular}
		\end{small}
	\end{center}
	\renewcommand*\arraystretch{1.0}
	\vspace{-0.1in}
\end{table*}

In the experiments, the features have been normalized fairly, i.e., $\x_t\leftarrow \x_t/\|\x_t\|$, which is reasonable since instances are received sequentially in online learning setting. Each dataset has been randomly divided into 5 folds, in which 4 folds are for training and the remaining fold is for testing. We also generate 4 independent 5-fold partitions per dataset to further reduce the effects of random partition on the algorithms. Therefore, the reported AUC values are the average results of 20 runs for each dataset. 5-fold cross validation is conducted on the training sets to decide on the learning rate $\eta\in2^{[-10:10]}$ and the regularization parameter $\lambda\in2^{[-10:6]}$. For OAM$_{gra}$ and OAM$_{seq}$, the buffer size is fixed at 100 as suggested in~\cite{DBLP:conf/icml/ZhaoHJY11}. All experiments for online setting comparisons were conducted with MATLAB on a computer workstation with 16GB memory and 3.20GHz CPU. On the other hand, for fair comparisons in batch settings, the core steps of the algorithms were implemented in C++ since we directly use the respective toolboxes ~\footnote{\url{http://www.cs.cornell.edu/people/tj/svm_light/svm_perf.html}}~\footnote{\url{http://lamda.nju.edu.cn/code_CAPO.ashx}} provided by the respective authors of the SVM-perf and CAPO algorithms.

\subsection{Evaluation of AdaOAM on Benchmark Datasets}

Table~\ref{table:AUC-value1} summarizes the average AUC performance of the algorithms under studied over the 16 datasets for online setting. In this table, we use $\bullet/\circ$ to indicate that AdaOAM is significantly better/worse than the corresponding method (pairwise $t$-tests at 95\% significance level).
\begin{table*}[!htpb]
	\vspace{-0.1in}
	\renewcommand*\arraystretch{1.1}
	\begin{center}
		\caption{AUC performance evaluation (mean$\pm$std.) of AdaOAM against other online algorithms on benchmark datasets. $\bullet/\circ$ indicates that AdaOAM is significantly better/worse than the corresponding method (pairwise $t$-tests at 95\% significance level).}\label{table:AUC-value1}
		\vspace{-0.1in}
		\begin{scriptsize}
			\begin{tabular}{|c|c|c|c|c|c|c|c|c|c|}        \hline
				datasets       & AdaOAM  & OPAUC & OAM$_{seq}$ & OAM$_{gra}$ & online Uni-Log & online Uni-Exp  \\
				\hline
				glass          &.816 	$\pm$ .058 &.804 	$\pm$ .059$\bullet$&.808 	$\pm$ .086$\bullet$&.817 	$\pm$ .057$\circ$  &.797 	$\pm$ .074$\bullet$&.795 	$\pm$ .074$\bullet$ \\
				heart          &.912 	$\pm$ .029 &.910 	$\pm$ .029$\bullet$&.907 	$\pm$ .027$\bullet$&.892 	$\pm$ .030$\bullet$&.906 	$\pm$ .030$\bullet$&.908 	$\pm$ .029$\bullet$ \\
				svmguide4      &.819 	$\pm$ .089 &.740 	$\pm$ .077$\bullet$&.782 	$\pm$ .051$\bullet$&.781 	$\pm$ .069$\bullet$&.588 	$\pm$ .106$\bullet$&.583 	$\pm$ .101$\bullet$ \\
				liver-disorders&.719 	$\pm$ .034 &.711 	$\pm$ .036$\bullet$&.704 	$\pm$ .047$\bullet$&.693 	$\pm$ .060$\bullet$&.695 	$\pm$ .042$\bullet$&.698 	$\pm$ .040$\bullet$ \\
				balance        &.579 	$\pm$ .106 &.551 	$\pm$ .114$\bullet$&.524 	$\pm$ .098$\bullet$&.508 	$\pm$ .078$\bullet$&.385 	$\pm$ .091$\bullet$&.385 	$\pm$ .090$\bullet$ \\
				breast         &.992 	$\pm$ .005 &.992 	$\pm$ .006         &.989 	$\pm$ .008         &.992 	$\pm$ .006         &.992 	$\pm$ .006         &.992 	$\pm$ .006          \\
				australian     &.927 	$\pm$ .016 &.926 	$\pm$ .016         &.915 	$\pm$ .024$\bullet$&.911 	$\pm$ .021$\bullet$&.924 	$\pm$ .016         &.923 	$\pm$ .017          \\
				diabetes       &.826 	$\pm$ .031 &.825 	$\pm$ .031         &.820 	$\pm$ .030$\bullet$&.808 	$\pm$ .046$\bullet$&.824 	$\pm$ .032         &.824 	$\pm$ .032          \\
				vehicle        &.818 	$\pm$ .026 &.816 	$\pm$ .025         &.815 	$\pm$ .027         &.792 	$\pm$ .035$\bullet$&.770 	$\pm$ .031$\bullet$&.774 	$\pm$ .031$\bullet$ \\
				german         &.771 	$\pm$ .031 &.730 	$\pm$ .052$\bullet$&.751 	$\pm$ .044$\bullet$&.730 	$\pm$ .033$\bullet$&.662 	$\pm$ .037$\bullet$&.663 	$\pm$ .037$\bullet$ \\
				svmguide3      &.734 	$\pm$ .038 &.724 	$\pm$ .038$\bullet$&.719 	$\pm$ .041$\bullet$&.696 	$\pm$ .047$\bullet$&.674 	$\pm$ .041$\bullet$&.678 	$\pm$ .042$\bullet$ \\
				a2a            &.873 	$\pm$ .019 &.872 	$\pm$ .021$\bullet$&.840 	$\pm$ .022$\bullet$&.839 	$\pm$ .019$\bullet$&.862 	$\pm$ .019$\bullet$&.866 	$\pm$ .020$\bullet$ \\
				magic04        &.798 	$\pm$ .007 &.765 	$\pm$ .009$\bullet$&.777 	$\pm$ .012$\bullet$&.752 	$\pm$ .023$\bullet$&.754 	$\pm$ .008$\bullet$&.752 	$\pm$ .008$\bullet$ \\
				cod-rna        &.962 	$\pm$ .002 &.919 	$\pm$ .003$\bullet$&.942 	$\pm$ .005$\bullet$&.936 	$\pm$ .004$\bullet$&.919 	$\pm$ .003$\bullet$&.919 	$\pm$ .003$\bullet$ \\
				acoustic       &.894 	$\pm$ .002 &.888    $\pm$ .002$\bullet$&.882 	$\pm$ .048$\bullet$&.871 	$\pm$ .054$\bullet$&.768 	$\pm$ .069$\bullet$&.796 	$\pm$ .196$\bullet$ \\
				poker          &.521 	$\pm$ .007 &.520 	$\pm$ .007         &.507 	$\pm$ .008$\bullet$&.508 	$\pm$ .016$\bullet$&.488 	$\pm$ .006$\bullet$&.488 	$\pm$ .006$\bullet$ \\\hline
				\multicolumn{2}{|c|}{win/tie/loss}       &11/5/0                     &14/2/0                     &14/1/1                     &13/3/0                     &13/3/0                 \\
				\hline
			\end{tabular}
		\end{scriptsize}
	\end{center}
	\renewcommand*\arraystretch{0.9}
	\vspace{-0.1in}
\end{table*}

From the results in Table~\ref{table:AUC-value1}, several interesting observations can be drawn. Firstly, the win/tie/loss counts show that the AdaOAM is clearly superior to the counterpart algorithms considered for comparison, as it wins in most cases and has zero loses in terms of AUC performance. This indicates that the proposed AdaOAM is the most effective online AUC optimization algorithm among all others considered. Secondly, AdaOAM outperforms the first-order online AUC maximization algorithms, including, OPAUC, OAM$_{seq}$, and OAM$_{gra}$, thus demonstrating that second-order information can help significantly improve the learning efficacy of existing online AUC optimization algorithms. In addition, on svmguide4, balance scale, and poker hand datasets, the optimization methods based on pairwise loss functions including AdaOAM, OPAUC, OAM$_{seq}$, and OAM$_{gra}$ perform far better than those methods based on univariate loss functions including Uni-Log and Uni-Exp. This highlights the significance and effectiveness of methods based on pairwise loss function optimization over univariate loss function ones.

To study the efficiency of the proposed AdaOAM algorithm, Figure~\ref{fig/time} depicts the running time (in milliseconds) of AdaOAM versus other online learning algorithms on all the 16 benchmark datasets.

\begin{figure*}[htbp]\vspace{-0.1in}
	\centering\includegraphics[width=7in]{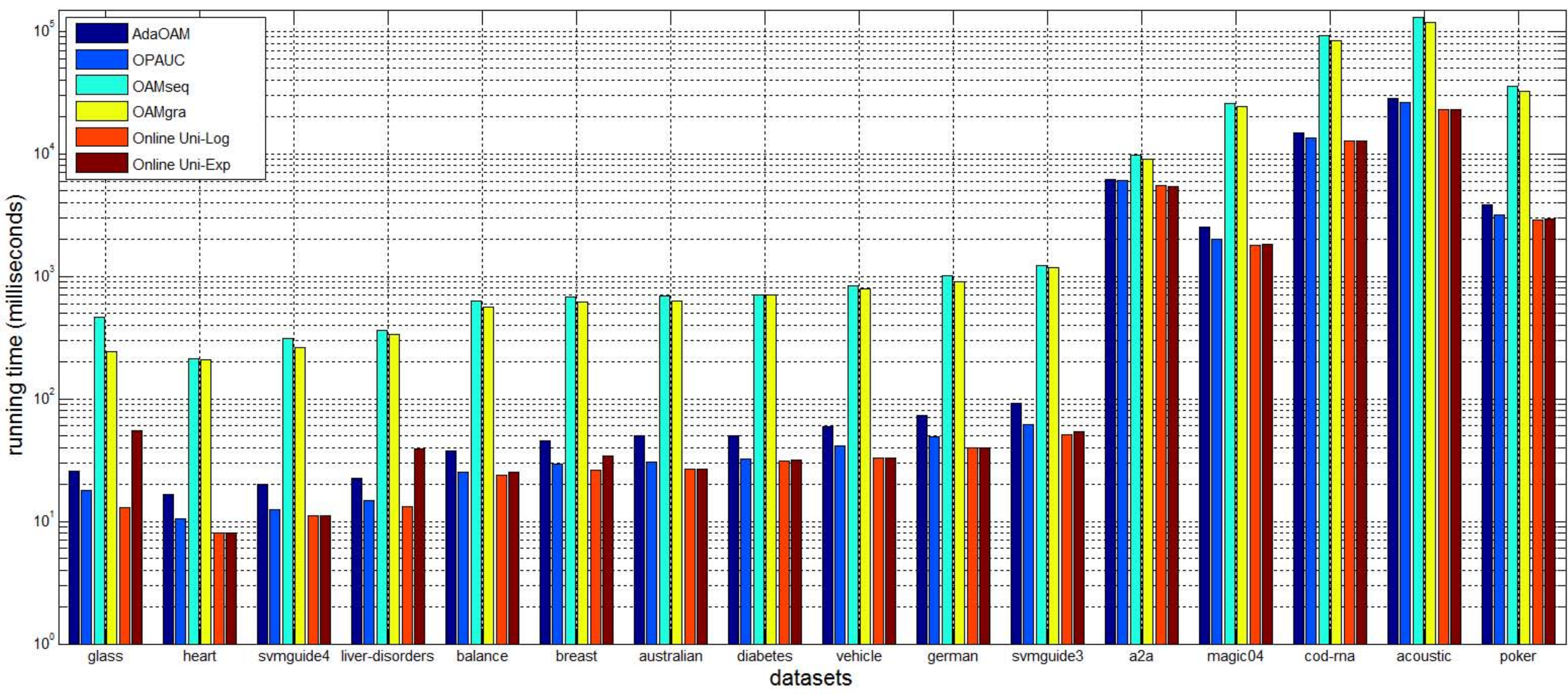}\vspace{-0.2in}
	\caption{Running time (in milliseconds) of AdaOAM and the other online learning algorithms on all 16 benchmark datasets. Note that the \emph{y}-axis is presented in log-scale.}\label{fig/time}\vspace{-0.1in}
\end{figure*}

From the results in Figure~\ref{fig/time}, it can be observed that the empirical computational complexity of AdaOAM is in general comparable to the other online learning algorithms, while being more efficient than OAM$_{seq}$ and OAM$_{gra}$ on some of the datasets, such as, glass, heart, etc. This indicates that the proposed algorithm is scalable and efficient, making it more attractive to large-scale real-world applications.

Next we move from the online setting to a batch learning mode. In particular, Table~\ref{table:AUC-value2} summarizes the average AUC performance of the algorithms under comparison over the 16 datasets in a batch setting. Note that here $\bullet/\circ$ is used to indicate that AdaOAM is significantly better/worse than the corresponding method.
\begin{table*}[!htpb]
	\vspace{-0.1in}
	\renewcommand*\arraystretch{1.1}
	\begin{center}
		\caption{AUC performance evaluation (mean$\pm$std.) of AdaOAM against other batch algorithms on benchmark datasets. $\bullet/\circ$ indicates that AdaOAM is significantly better/worse than the corresponding method (pairwise $t$-tests at 95\% significance level).}\label{table:AUC-value2}
		\vspace{-0.1in}
		\begin{scriptsize}
			\begin{tabular}{|c|c|c|c|c|c|c|c|}        \hline
				datasets       & AdaOAM  & SVM-perf & CAPO & batch Uni-Log & batch Uni-Squ  \\
				\hline
				glass          &.816 	$\pm$ .058 &.822	$\pm$ .060$\circ$  &.839	$\pm$ .057$\circ$  &.807 	$\pm$ .071$\bullet$&.819 	$\pm$ .060          \\
				heart          &.912 	$\pm$ .029 &.921	$\pm$ .016$\circ$  &.923	$\pm$ .013$\circ$  &.900 	$\pm$ .034$\bullet$&.902 	$\pm$ .034$\bullet$ \\
				svmguide4      &.819 	$\pm$ .089 &.871	$\pm$ .048$\circ$  &.841	$\pm$ .022$\circ$  &.596 	$\pm$ .100$\bullet$&.695 	$\pm$ .091$\bullet$ \\
				liver-disorders&.719 	$\pm$ .034 &.722	$\pm$ .084         &.729	$\pm$ .045$\circ$  &.703 	$\pm$ .052$\bullet$&.693 	$\pm$ .057$\bullet$ \\
				balance        &.579 	$\pm$ .106 &.592	$\pm$ .025$\circ$  &.620	$\pm$ .010$\circ$  &.355 	$\pm$ .081$\bullet$&.517 	$\pm$ .206$\bullet$ \\
				breast         &.992 	$\pm$ .005 &.998	$\pm$ .069         &.995	$\pm$ .044         &.995 	$\pm$ .004         &.994 	$\pm$ .005          \\
				australian     &.927 	$\pm$ .016 &.939	$\pm$ .052$\circ$  &.915	$\pm$ .030$\bullet$&.924 	$\pm$ .016         &.925 	$\pm$ .016          \\
				diabetes       &.826 	$\pm$ .031 &.836	$\pm$ .020$\circ$  &.852	$\pm$ .024$\circ$  &.828 	$\pm$ .031         &.828 	$\pm$ .032          \\
				vehicle        &.818 	$\pm$ .026 &.820	$\pm$ .034         &.782	$\pm$ .029$\bullet$&.823 	$\pm$ .029$\circ$  &.798 	$\pm$ .029$\bullet$ \\
				german         &.771 	$\pm$ .031 &.777	$\pm$ .043         &.783	$\pm$ .033$\circ$  &.766 	$\pm$ .035$\bullet$&.734 	$\pm$ .034$\bullet$ \\
				svmguide3      &.734 	$\pm$ .038 &.753	$\pm$ .056$\circ$  &.752	$\pm$ .055$\circ$  &.743 	$\pm$ .032$\circ$  &.729 	$\pm$ .037$\bullet$ \\
				a2a            &.873 	$\pm$ .019 &.881	$\pm$ .012$\circ$  &.865	$\pm$ .049$\bullet$&.878 	$\pm$ .017$\circ$  &.873 	$\pm$ .019          \\
				magic04        &.798 	$\pm$ .007 &.792	$\pm$ .031$\bullet$&.780	$\pm$ .045$\bullet$&.758 	$\pm$ .008$\bullet$&.757 	$\pm$ .008$\bullet$ \\
				cod-rna        &.962 	$\pm$ .002 &.954	$\pm$ .027$\bullet$&.986	$\pm$ .073$\circ$  &.949 	$\pm$ .003$\bullet$&.924 	$\pm$ .003$\bullet$ \\
				acoustic       &.894 	$\pm$ .002 &.897	$\pm$ .038         &.892	$\pm$ .075         &.878 	$\pm$ .007$\bullet$&.865 	$\pm$ .006$\bullet$ \\
				poker          &.521 	$\pm$ .007 &.524	$\pm$ .070         &.517	$\pm$ .010$\bullet$&.497 	$\pm$ .024$\bullet$&.496 	$\pm$ .006$\bullet$ \\  \hline
				\multicolumn{2}{|c|}{win/tie/loss}       &2/6/8                     &5/2/9                     &10/3/3                     &11/5/0                            \\
				\hline
			\end{tabular}
		\end{scriptsize}
	\end{center}
	\renewcommand*\arraystretch{1.0}
	\vspace{-0.1in}
\end{table*}

From Table~\ref{table:AUC-value2}, the following observations have been observed. Firstly, the win/tie/loss counts show that the AdaOAM is superior to batch Uni-Log and batch Uni-Squ in many cases. Since batch Uni-Log and batch Uni-Squ operate on optimizing univariate loss functions, the results demonstrates the significance of adopting pair-wise loss function for AUC maximization. Secondly, the performance of AdaOAM is competitive to CAPO, but underperformed SVM-perf, which is expected since AdaOAM trades efficacy with efficiency.
%More importantly, AdaOAM can obtain comparable AUC performance, compared with CAPO and SVM-perf, under some cases, which implies the effectiveness of AdaOAM.

To analyze the efficiency of the proposed AdaOAM algorithm, we summarize the running time (in milliseconds) of AdaOAM and the other batch learning algorithms on all the benchmark datasets in Figure~\ref{fig/time_batch}. Since the core steps of the SVM-perf and CAPO are implemented based on the specific toolboxes~\footnote{\url{http://www.cs.cornell.edu/people/tj/svm_light/svm_perf.html}}~\footnote{\url{http://lamda.nju.edu.cn/code_CAPO.ashx}} developed by the respective authors in C++, we have implemented the core steps of the AdaOAM, batch Uni-Log, and batch Uni-Squ in C++ in our work to obtain the time cost comparisons reported here.

\begin{figure*}[htbp]\vspace{-0.1in}
	\centering\includegraphics[width=7in]{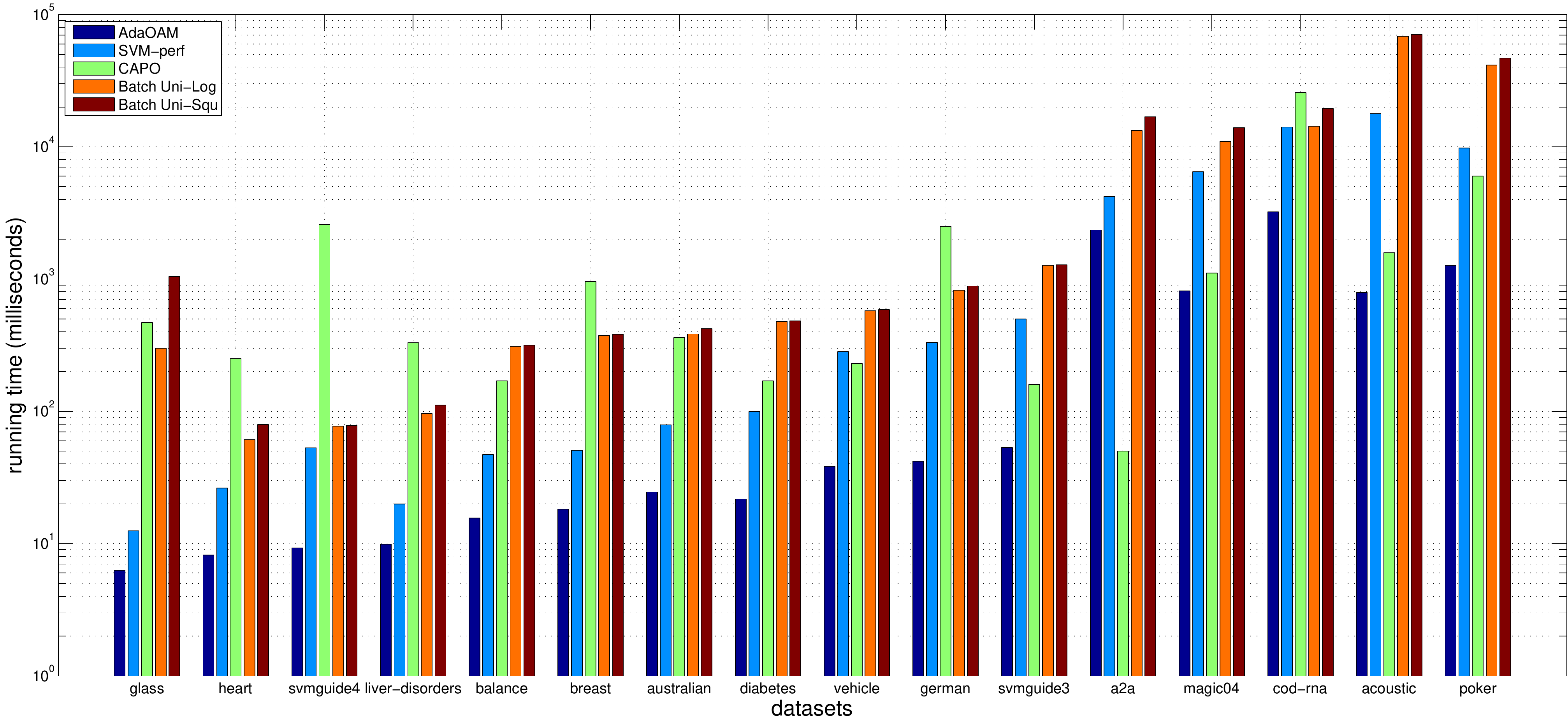}\vspace{-0.2in}
	\caption{Running time (in milliseconds) of AdaOAM and the other batch learning algorithms on all 16 benchmark datasets. Note that the \emph{y}-axis is presented in log-scale.}\label{fig/time_batch}\vspace{-0.1in}
\end{figure*}

Several core observations can be drawn from Figure~\ref{fig/time_batch}. First of all, the results on time costs highlights the higher efficiency of the AdaOAM against the other batch learning algorithms in general. Second, the empirical computational time cost of AdaOAM is noted to be significantly lower than the batch Uni-Log and batch Uni-Squ, which is attributed to the difference between an online setting and a batch setting. Third, CAPO is noted to be less efficient than SVM-perf in most cases. This is because CAPO being an ensemble learning method for AUC optimization is expected to incur higher training time. Besides, SVM-perf and CAPO need more time costs than that of the AdaOAM especially on large datasets although both of them are designed to be fast algorithms for performance measure optimization tasks. One exception from the results where the time cost of AdaOAM on the ``a2a" dataset is observed to be higher than that of the CAPO method. The reason is that ``a2a" dataset being a highly sparse dataset is a clear advantage for the CAPO since it operates under the framework of Core Vector Machine method~\cite{DBLP:journals/jmlr/TsangKC05}, which is a very fast batch algorithm for training SVM model involving sparse data. On the other hand, AdaOAM is not optimally designed to deal with sparse dataset and since it need to compute the covariance matrices when updating the model, and at the same time C++ programming is not appropriate and efficient to deal with matrix computations.

\subsection{Evaluation of Online Performance}
Next we study the online performance of AdaOAM versus the other online learning algorithms and highlight 6 datasets for illustration. Specifically, Figure~\ref{fig:convergence}(a)-(h) report the average AUC performance (across 20 independent runs) of the online model on the testing datasets. From the results, AdaOAM is once again shown to significantly outperform all the other three counterparts in the online learning process, which is consistent to our theoretical analysis that AdaOAM can more effectively exploit second order information to achieve improved regret bounds and robust performance.

\begin{figure*}[!htbp]
	\vspace{-0.1in}
	\centering
	\subfigure[heart]{
		\includegraphics[width=2.2in]{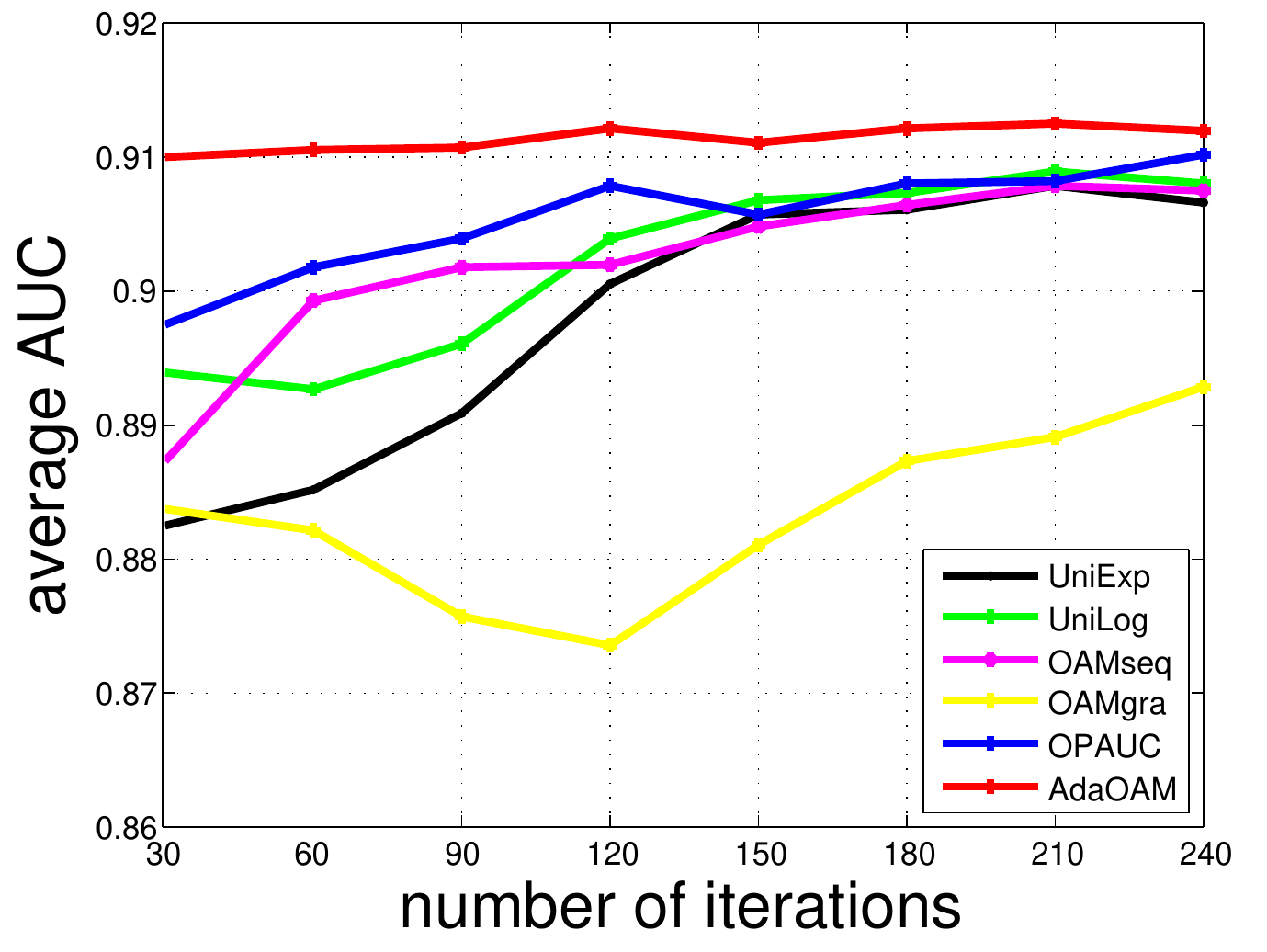}
	}\hspace{0.1in}
	\subfigure[svmguide4]{
		\includegraphics[width=2.2in]{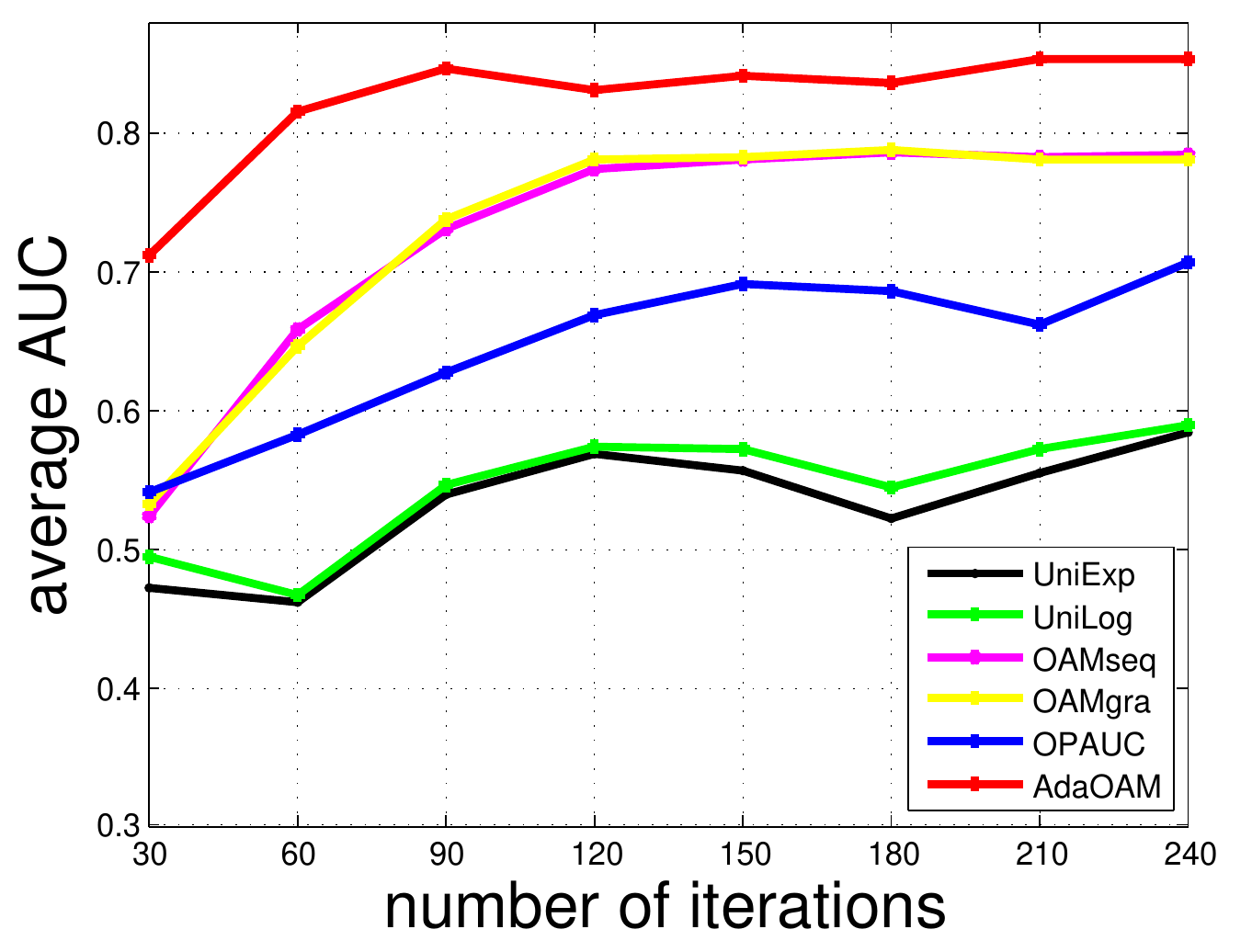}
	}\hspace{0.1in}
	\subfigure[liver-disorders]{
		\includegraphics[width=2.2in]{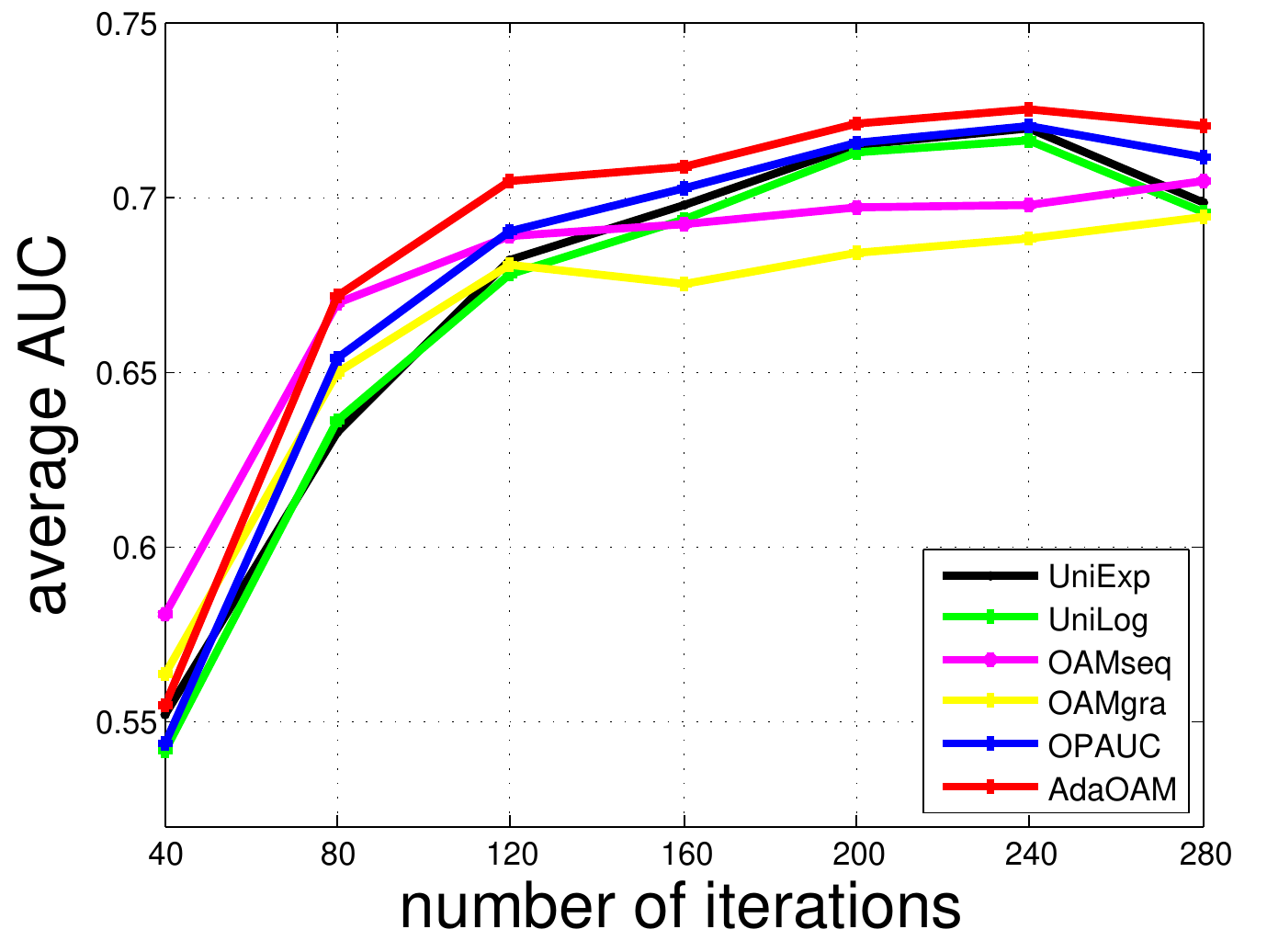}
	}\\ \vspace{-0.1in}
	\subfigure[australian]{
		\includegraphics[width=2.2in]{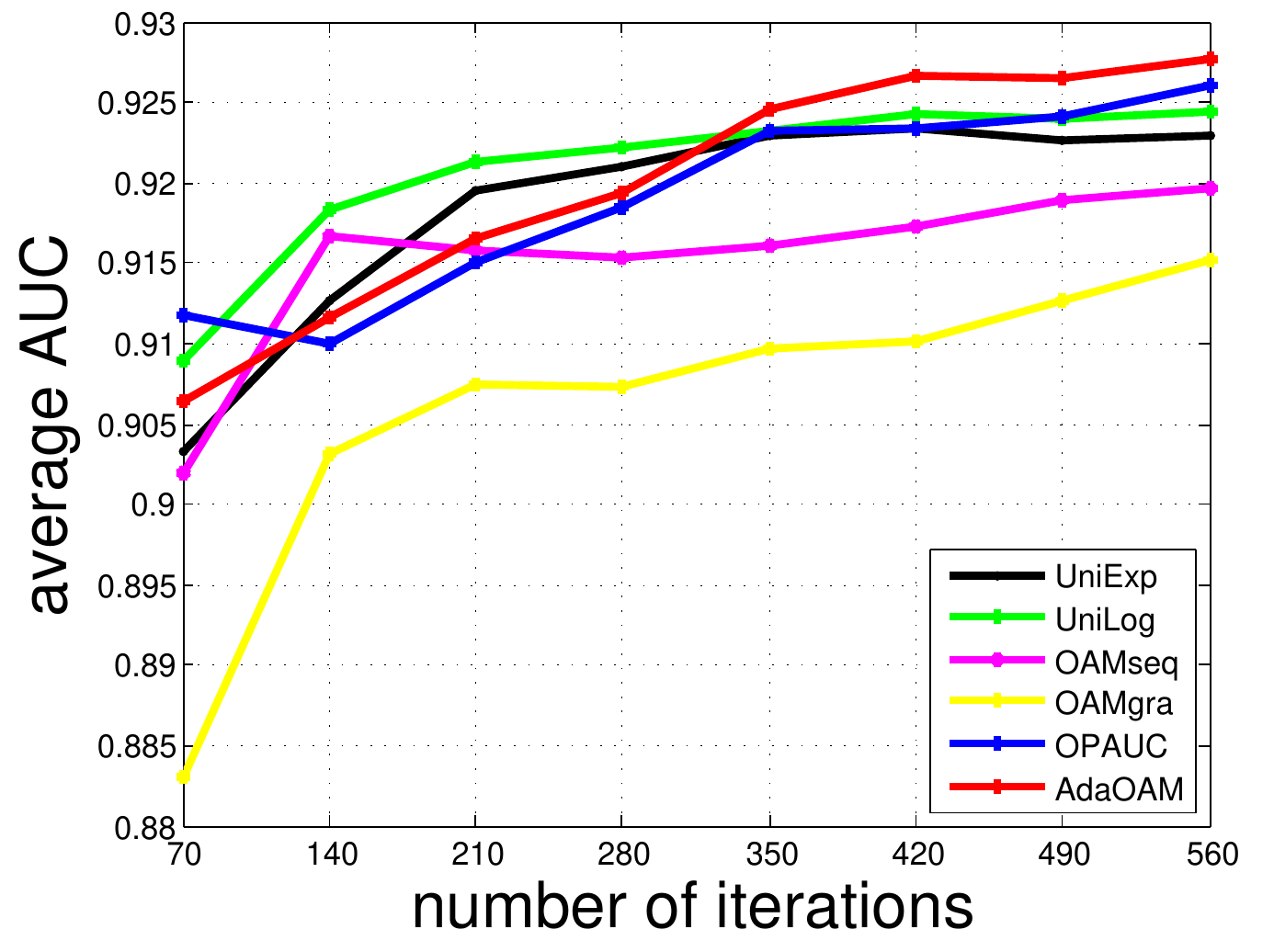}
	}\hspace{0.1in}
	\subfigure[german]{
		\includegraphics[width=2.2in]{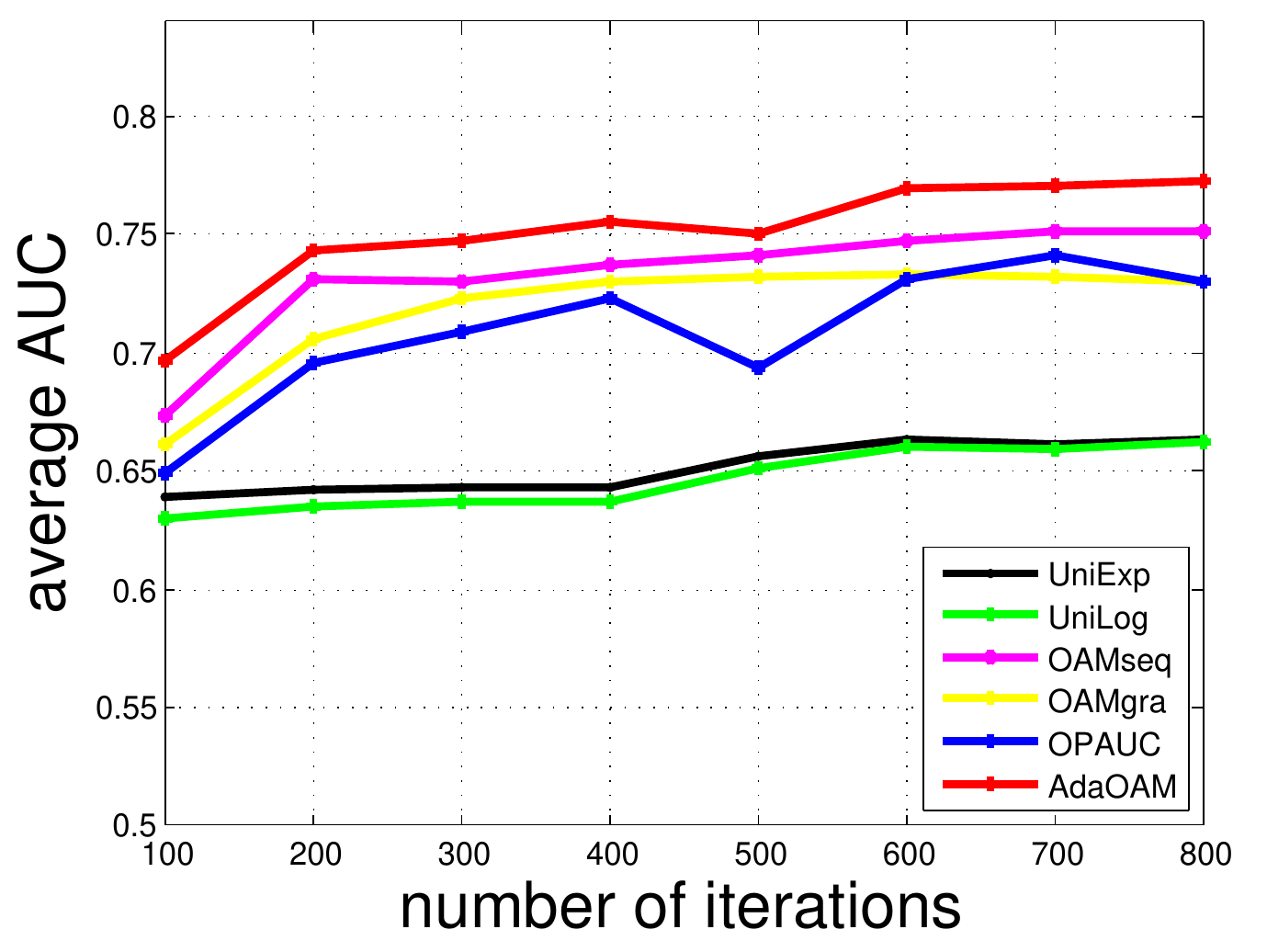}
	}\hspace{0.1in}
	\subfigure[svmguide3]{
		\includegraphics[width=2.2in]{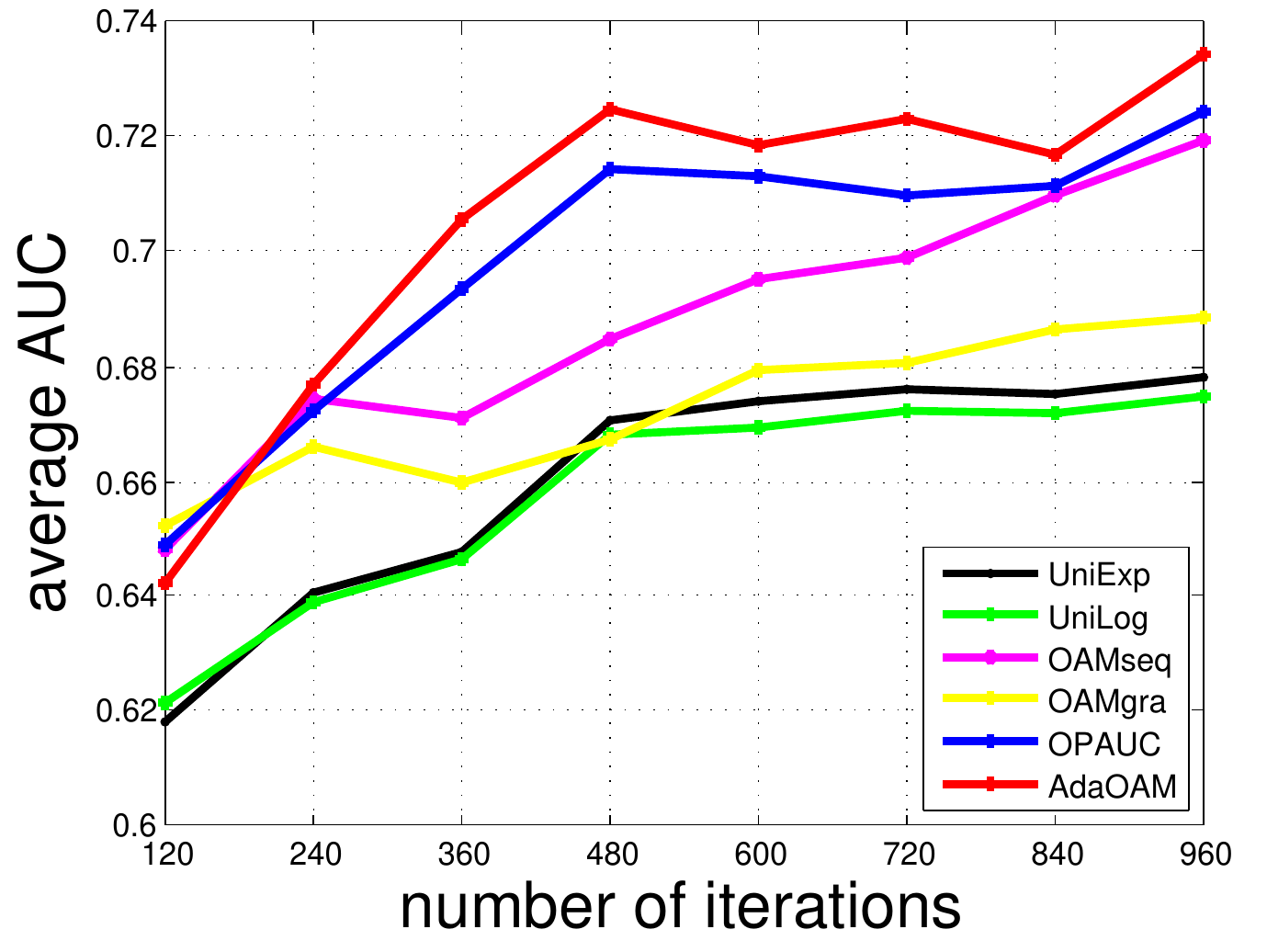}
	}\hspace{0.1in}
%	\subfigure[magic04]{
%		\includegraphics[width=2.5in]{fig/magic04.pdf}
%	}\hspace{0.1in}
%	\subfigure[poker]{
%		\includegraphics[width=2.5in]{fig/poker.pdf}
%	}
	\\ \vspace{-0.1in}
	\centering
	\caption{Evaluation of convergence rate on benchmark datasets.} \label{fig:convergence}
\end{figure*}

\subsection{Evaluation of Parameter Sensitivity}
In this subsection, we proceed to examine the parameter sensitivity of the AdaOAM algorithm. In our study, we experimented the AdaOAM with a set of different learning rates that lies in the wide range of $\eta\in2^{[-8:4]}$. The average test AUC results of AdaOAM across the wide range of learning rates after a single pass through the training data of the respective benchmark datasets are then summarized in Figure~\ref{fig:parameter}(a)-(h). Due to the space constraints, the results of 8 datasets are reported here for illustrations. Since the AdaOAM algorithm provides a per-feature adaptive learning rate at each iteration, it is less sensitive to the learning rate $\eta$ than the standard SGD.

\begin{figure*}[htbp]
	\vspace{-0.1in}
	\centering
	\subfigure[heart]{
		\includegraphics[width=2.2in]{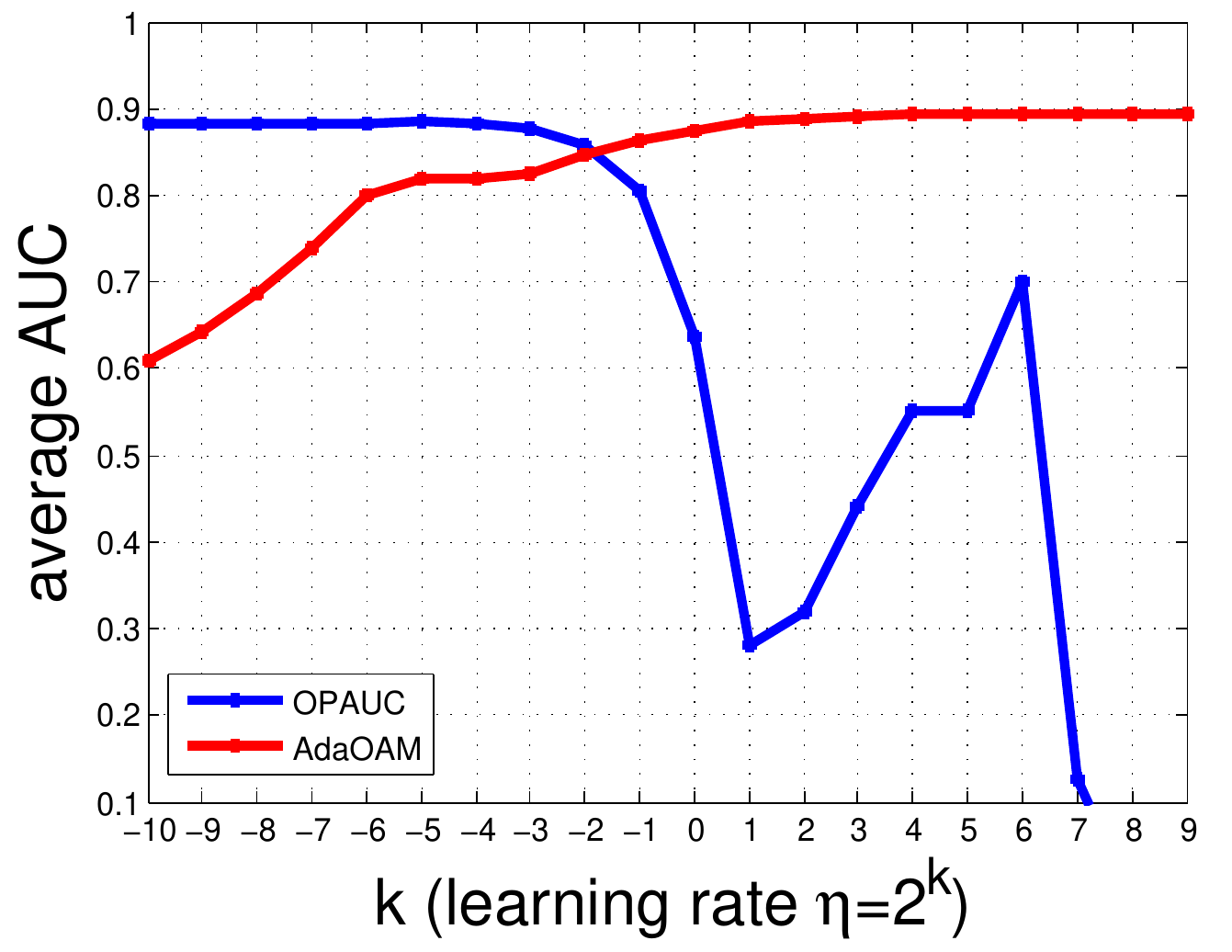}
	}\hspace{0.1in}
	\subfigure[svmguide4]{
		\includegraphics[width=2.2in]{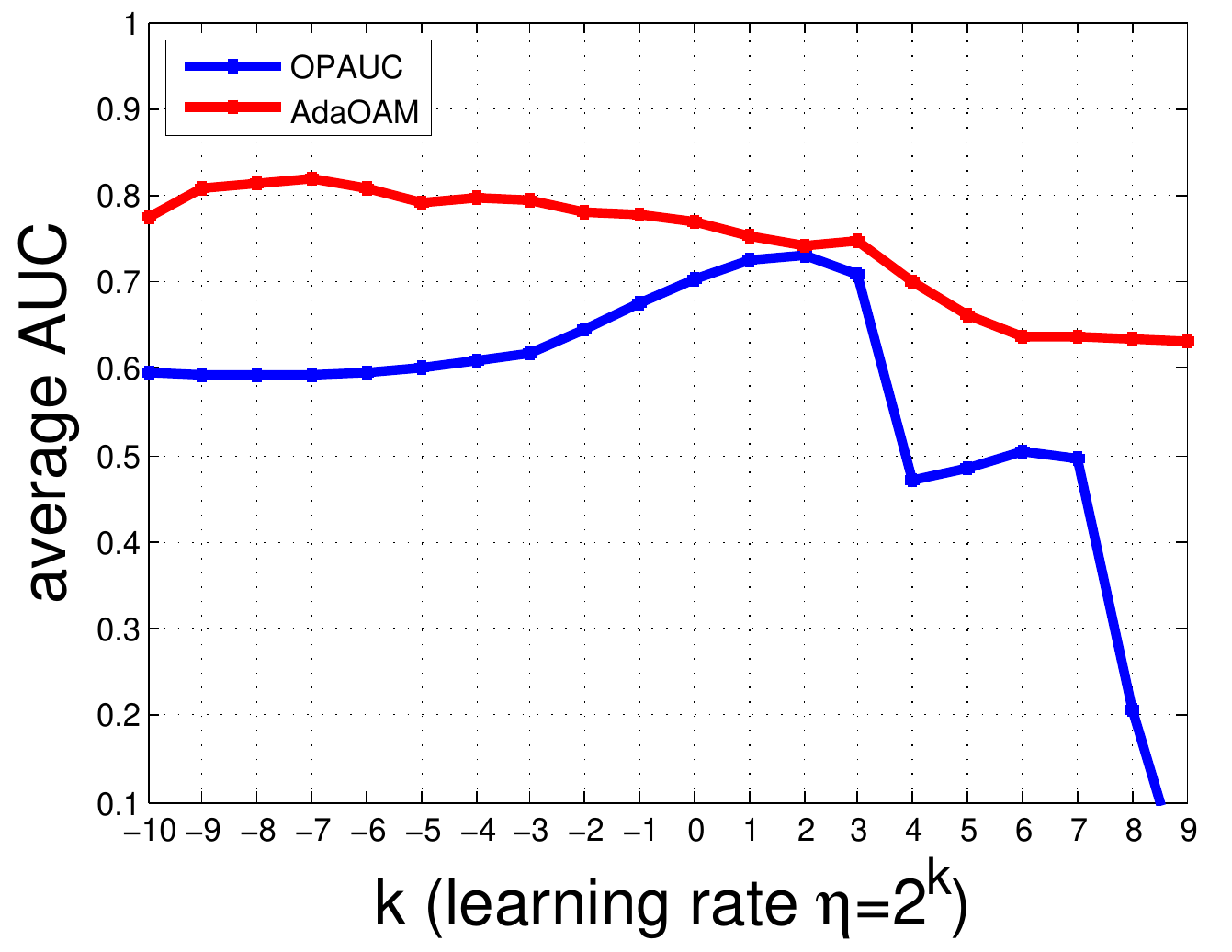}
	}\hspace{0.1in}
	\subfigure[liver-disorders]{
		\includegraphics[width=2.2in]{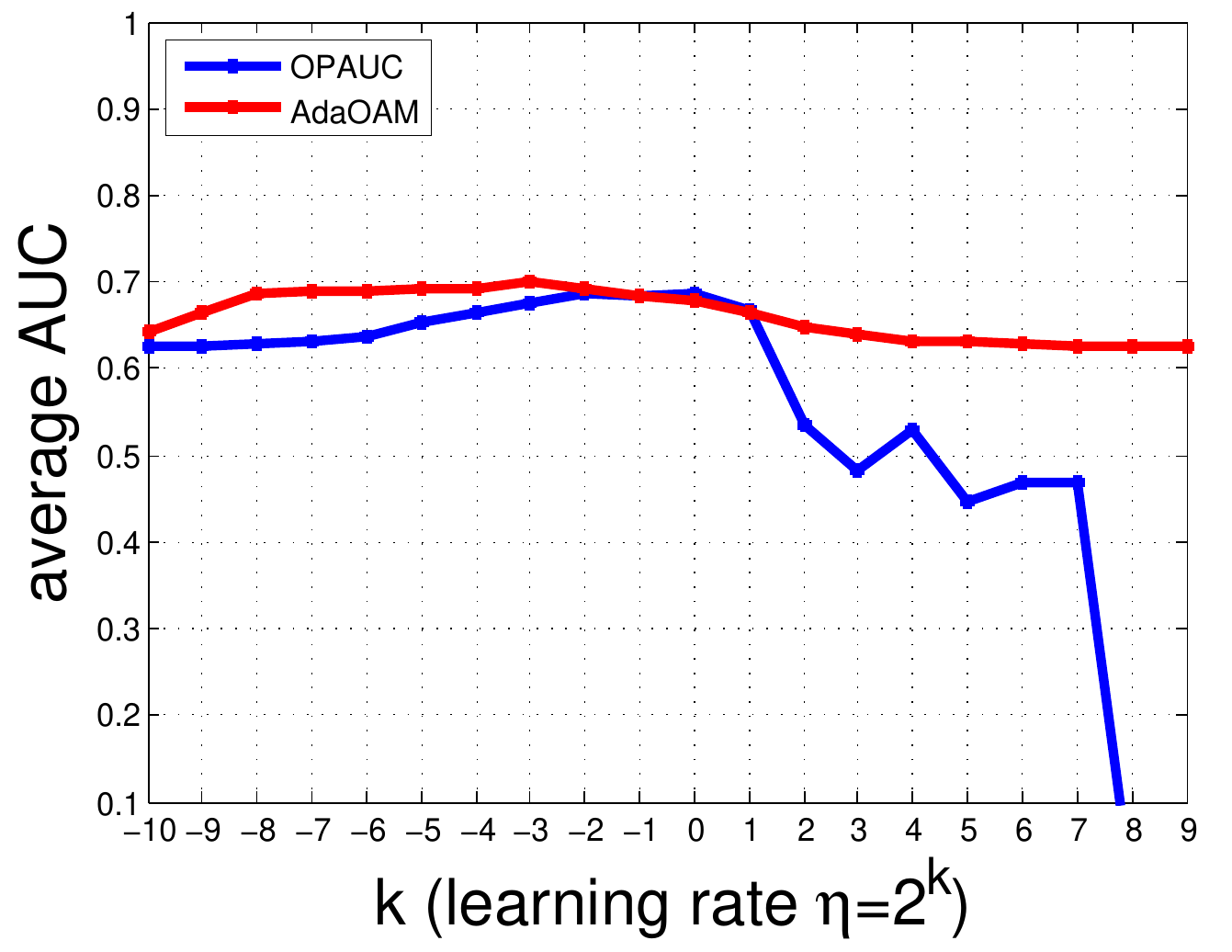}
	} \\ \vspace{-0.1in}
	\subfigure[diabetes]{
		\includegraphics[width=2.2in]{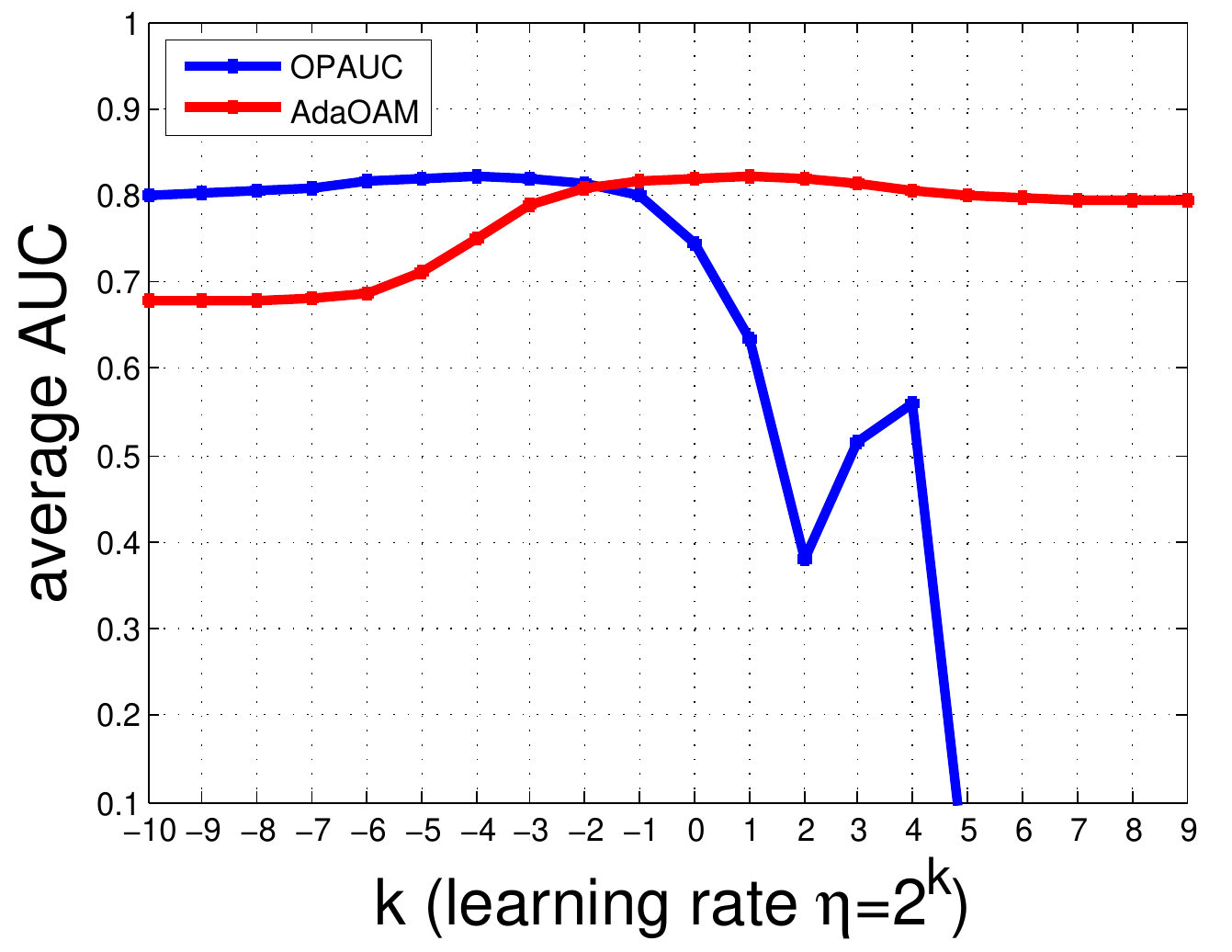}
	}\vspace{-0.1in}
	\subfigure[vehicle]{
		\includegraphics[width=2.2in]{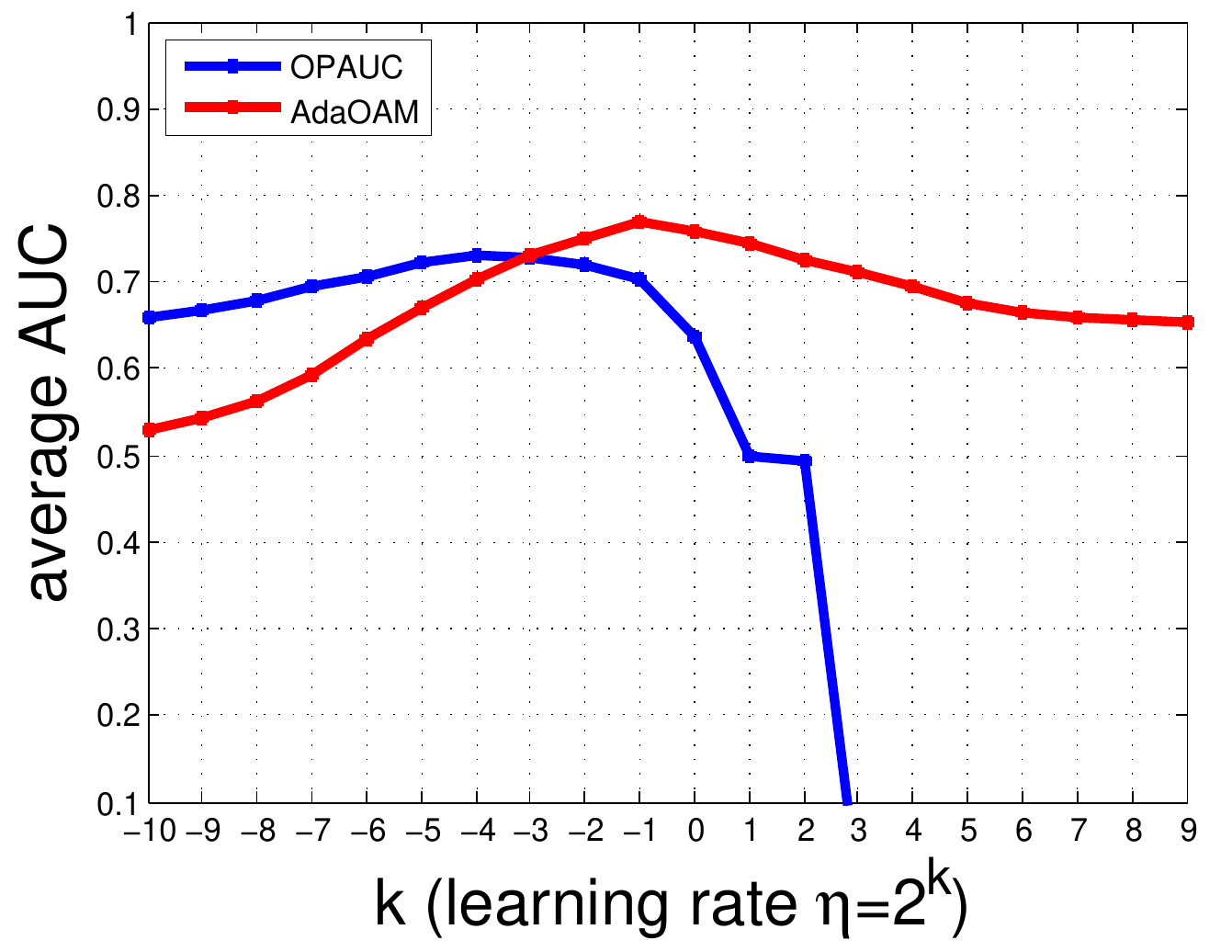}
	}\hspace{0.1in}
	\subfigure[german]{
		\includegraphics[width=2.2in]{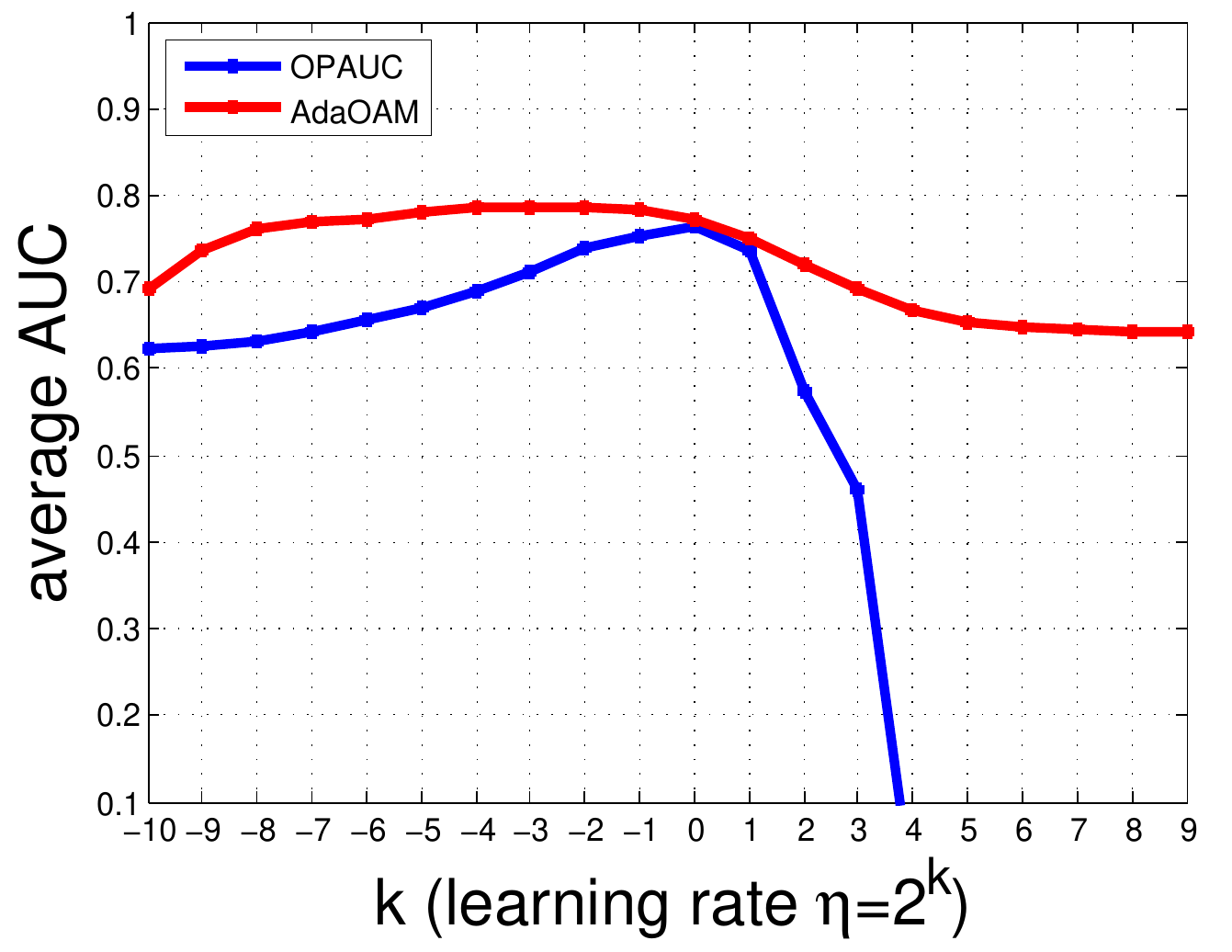}
	}\\ \vspace{-0.1in}
%	\subfigure[svmguide3]{
%		\includegraphics[width=2.5in]{fig/eta_svmguide3.pdf}
%	}\hspace{0.1in}
%	\subfigure[poker]{
%		\includegraphics[width=2.5in]{fig/eta_poker.pdf}
%	}
%	\vspace{-0.1in}
	\caption{Parameter sensitivity on benchmark datasets.} \label{fig:parameter}
\end{figure*}

In \cite{DBLP:conf/icml/GaoJZZ13}, the authors claimed that OPAUC was insensitive to the parameter settings. From Figure~\ref{fig:parameter}, it can be observed that AdaOAM is clearly more robust or insensitive to the learning rate than the OPAUC. The updating strategy by OPAUC is based on simple SGD, which usually requires quite some efforts of tuning the learning rate parameter sufficiently. On the other hand, the adaptive gradient strategy of AdaOAM is theoretically sound for learning rate adaptation since it takes full advantage of the historical gradient information available in the learning process. As such, AdaOAM is less sensitive to the parameter settings. Moreover, AdaOAM exhibits a natural phenomena of decreasing learning rate with increasing iterations.

\subsection{Evaluation of SAdaOAM on High-dimensional Sparse Datasets}

In this subsection, we move on to evaluate the empirical performance of various online AUC maximization algorithms on the publicly available high-dimensional sparse datasets as summarized in Table~\ref{table:sparse}.
The \textbf{pcmac} dataset is downloaded from SVMLin~\footnote{\url{http://vikas.sindhwani.org/svmlin.html}}. The \textbf{farm ads} dataset is generated based on the text ads found on twelve websites that deal with various farm animal related topics~\cite{Mesterharm11activelearning} and downloaded from UCI Machine Learning Repository. The \textbf{Reuters} dataset is the ModApte version~\footnote{\url{http://www.cad.zju.edu.cn/home/dengcai/Data/TextData.html}} upon removing documents with multiple category labels, and contains 8293 documents in 65 categories. The \textbf{sector}, \textbf{rcv1m}, and \textbf{news20} datasets are taken from the LIBSVM dataset website. Note that the original \textbf{sector}, \textbf{Reuters}, \textbf{rcv1m}, and \textbf{news20} are multi-class datasets; in our experiments, we randomly group the multiple classes into two meta-class in which each contains the same number of classes.

\begin{table}[h]
	\vspace{-0.1in}
	\renewcommand*\arraystretch{0.9}
	\begin{center}
		\caption{\small Details of high sparse datasets.}\label{table:sparse}
		\vspace{-0.1in}
		\begin{small}
			\begin{tabular}{|c|cccc|}         \hline
				datasets &    $\#$ inst    & $\#$ dim  &$T_-/T_+$  &sparsity \\
				\hline\hline
				%kdd04\_phy       & 150,000   & 78      & 38.42\%   \\
				pcmac            & 1,946     & 7,510   & 1.0219  & 3.99\%  \\
				farm ads         & 4,143     & 54,877  & 1.1433  & 0.36\%  \\
				sector           & 9,619     & 55,197  & 1.0056  & 0.29\%  \\
				Reuters          & 8,293     & 18,933  & 42.6474 & 0.25\%  \\
				rcv1m            & 20,242    & 47,236  & 1.0759  & 0.16\%  \\
				news20           & 15,935    & 62,061  & 1.0042  & 0.13\%
				\\\hline
			\end{tabular}
		\end{small}
	\end{center}
	\renewcommand*\arraystretch{0.9}
	\vspace{-0.2in}
\end{table}

In consistent to the previous experimental settings, we conduct 5 fold cross validation on the training sets to identify the most appropriate learning rate $\eta\in2^{[-8:4]}$ and sparse regularization parameter $\theta\in10^{[-8:-1]}$. We fix the parameter $\lambda$ as $10^{-6}$ since its effect on the learning performance is negligible, especially if it is sufficiently small. The performance of all the algorithms are evaluated across 4 independent trials of 5 fold cross validation, and then the reported AUC values are the average of the 20 runs. To showcase the benefit of the SAdaOAM algorithm on high-dimensional sparse datasets, the original AdaOAM is also included for comparison, in addition to the state-of-the-art online learning algorithms considered in the earlier sections. The experimental results of SAdaOAM and other online learning algorithms in terms of AUC performance on all the testing data are then reported in Table~\ref{table:AUC-sparse}.

\begin{table*}[!htpb]
	\vspace{-0.2in}
	\renewcommand*\arraystretch{1.1}
	\begin{center}
		\caption{AUC performance evaluation (mean$\pm$std.) of SAdaOAM versus other online algorithms on the high-dimensional sparse datasets. $\bullet/\circ$ indicates that SAdaOAM is significantly better/worse than the corresponding method (pairwise $t$-tests at 95\% significance level).}\label{table:AUC-sparse}
		\vspace{-0.1in}
		\begin{scriptsize}
			\begin{tabular}{|c|c|c|c|c|c|c|c|c|c|}        \hline
				datasets       & SAdaOAM  &AdaOAM & OPAUC & OAM$_{seq}$ & OAM$_{gra}$ & online Uni-Log & online Uni-Exp  \\
				\hline
				pcmac              &.953	$\pm$ .006 &.929 $\pm$ .008$\bullet$&.929 $\pm$ .013$\bullet$&.939	$\pm$ .012$\bullet$&.903	$\pm$ .028$\bullet$&.923	$\pm$ .024$\bullet$&.931	$\pm$ .010$\bullet$ \\
				farm ads           &.957 	$\pm$ .007 &.938 $\pm$ .010$\bullet$&.896 $\pm$ .033$\bullet$&.951  $\pm$ .007$\bullet$&.947 	$\pm$ .007$\bullet$&.939 	$\pm$ .008$\bullet$&.933 	$\pm$ .007$\bullet$ \\
				sector             &.950	$\pm$ .007 &.946 $\pm$ .011$\bullet$&.936 $\pm$ .008$\bullet$&.923 	$\pm$ .008$\bullet$&.918	$\pm$ .013$\bullet$&.894	$\pm$ .009$\bullet$&.929	$\pm$ .010$\bullet$ \\
				Reuters            &.940 	$\pm$ .030 &.933 $\pm$ .014$\bullet$&.908 $\pm$ .029$\bullet$&.926 	$\pm$ .022$\bullet$&.905 	$\pm$ .031$\bullet$&.891 	$\pm$ .017$\bullet$&.884 	$\pm$ .030$\bullet$ \\
				rcv1m              &.962 	$\pm$ .004 &.955 $\pm$ .002$\bullet$&.955 $\pm$ .004$\bullet$&.946 	$\pm$ .005$\bullet$&.906 	$\pm$ .011$\bullet$&.901 	$\pm$ .021$\bullet$&.920 	$\pm$ .020$\bullet$ \\
				news20             &.976	$\pm$ .019 &.977 $\pm$ .005         &.977 $\pm$ .007 	     &.967	$\pm$ .001$\bullet$&.936	$\pm$ .001$\bullet$&.935 	$\pm$ .003$\bullet$&.957 	$\pm$ .003$\bullet$ \\\hline
				\multicolumn{2}{|c|}{win/tie/loss}       &5/1/0                     &5/1/0                     &6/0/0                     &6/0/0                     &6/0/0    &6/0/0             \\
				\hline
			\end{tabular}
		\end{scriptsize}
	\end{center}
	\renewcommand*\arraystretch{1.1}
	\vspace{-0.2in}
\end{table*}

From the results in Table~\ref{table:AUC-sparse}, the following observations can be drawn. Firstly, all the pairwise loss function optimization algorithms have been observed to outperform the univariate loss function optimization algorithms i.e., Uni-Log and Uni-Exp, which stresses the benefits and high efficacy of pairwise loss function optimization for AUC maximization task. In addition, it is evident from the results that our proposed algorithm SAdaOAM is superior to the non-adaptive or non-sparse methods in most cases, which indicates our proposed algorithm with sparsity is potentially more effective than existing online AUC maximization algorithms that do not exploit the sparsity in the data.

To bring deeper insights on the mechanisms of the proposed algorithm, we take a further to analyze the sparsity level of the final learned model generated by SAdaOAM and the other online learning algorithms. The sparsity of the learned model plays a significant role in large-scale machine learning tasks in terms of both storage cost and efficiency. A sparse learned model not only speeds up the training process, but also reduces the storage cost requirements of large-scale systems. Here, we measure the sparsity level of a learned model based on the ratio of zero elements in the model and the results of the corresponding online algorithms are summarized in Table~\ref{table:AUC-sparsityCom}.

\begin{table*}[!htpb]
	\renewcommand*\arraystretch{1.2}
	\begin{center}
		\caption{Final learned model sparsity (\%) measure of SAdaOAM and the respective online learning algorithms. $\bullet/\circ$ indicates that SAdaOAM is significantly better/worse than the corresponding method (pairwise $t$-tests at 95\% significance level).}\label{table:AUC-sparsityCom}
		\vspace{-0.1in}
		\begin{scriptsize}
			\begin{tabular}{|c|c|c|c|c|c|c|c|c|c|}        \hline
				datasets       &SAdaOAM  &AdaOAM  & OPAUC & OAM$_{seq}$ & OAM$_{gra}$ & online Uni-Log & online Uni-Exp  \\
				\hline
				pcmac              &46.11	&00.03 $\bullet$&00.03	$\bullet$&08.21 $\bullet$&10.64 $\bullet$&00.70 $\bullet$&00.03 $\bullet$  \\
				farm ads           &86.09 	&10.79 $\bullet$&10.79	$\bullet$&46.19 $\bullet$&58.99 $\bullet$&10.79 $\bullet$&00.01 $\bullet$ \\
				sector             &73.78	&17.45 $\bullet$&17.11	$\bullet$&42.10 $\bullet$&62.52 $\bullet$&17.45 $\bullet$&17.45 $\bullet$ \\
				Reuters            &92.59 	&01.53 $\bullet$&01.53	$\bullet$&53.42 $\bullet$&59.96 $\bullet$&01.67 $\bullet$&00.01 $\bullet$ \\
				rcv1m              &84.72 	&28.91 $\bullet$&28.88	$\bullet$&79.06 $\bullet$&69.66 $\bullet$&29.00 $\bullet$&20.91 $\bullet$ \\
				news20             &72.61	&06.94 $\bullet$&06.94	$\bullet$&57.42 $\bullet$&71.09          &06.38 $\bullet$&06.38 $\bullet$\\\hline
				\multicolumn{2}{|c|}{win/tie/loss}       &6/0/0     &6/0/0    &6/0/0    &5/1/0      &6/0/0    &6/0/0    \\
				\hline
			\end{tabular}
		\end{scriptsize}
	\end{center}
	\renewcommand*\arraystretch{1.2}
	\vspace{-0.2in}
\end{table*}

Several observations can be drawn from Table~\ref{table:AUC-sparsityCom}. To begin with, we found that all the algorithms failed to produce high sparsity level solutions for online AUC maximization task except the proposed SAdaOAM algorithm. In contrast, the SAdaOAM approach gives consideration to both performance and sparsity level for all the cases. In particular, it is found that the higher the feature dimension is, the larger the sparsity is achieved by the proposed algorithm. To sum up, the SAdaOAM algorithm is an ideal method of online AUC optimization for those high-dimensional sparse data.

\subsection{Evaluation of SAdaOAM on Sparsity-AUC Tradeoffs}

In this set of experiments, we study the tradeoffs between the sparsity level and the AUC performance for the SAdaOAM algorithm. To achieve this, we set the regularization parameter for  $\ell_1$  with the range $\theta\in10^{[-8:0]}$ for the SAdaOAM method. We apply the same experimental setting executed above and record the average test AUC performance versus the sparsity level designated by the proportion of non-zeros in the final weight solution after a single pass through the training data. These experimental settings make the final model solutions ranging from an almost dense weight to a nearly all-zero one. We randomly choose three datasets for this set of experiments and show the results in Figure~\ref{fig:sparsity}.

\begin{figure*}[ht]
	\vspace{-0.1in}
	\centering
	\subfigure[pcmac]{
		\includegraphics[width=2.2in]{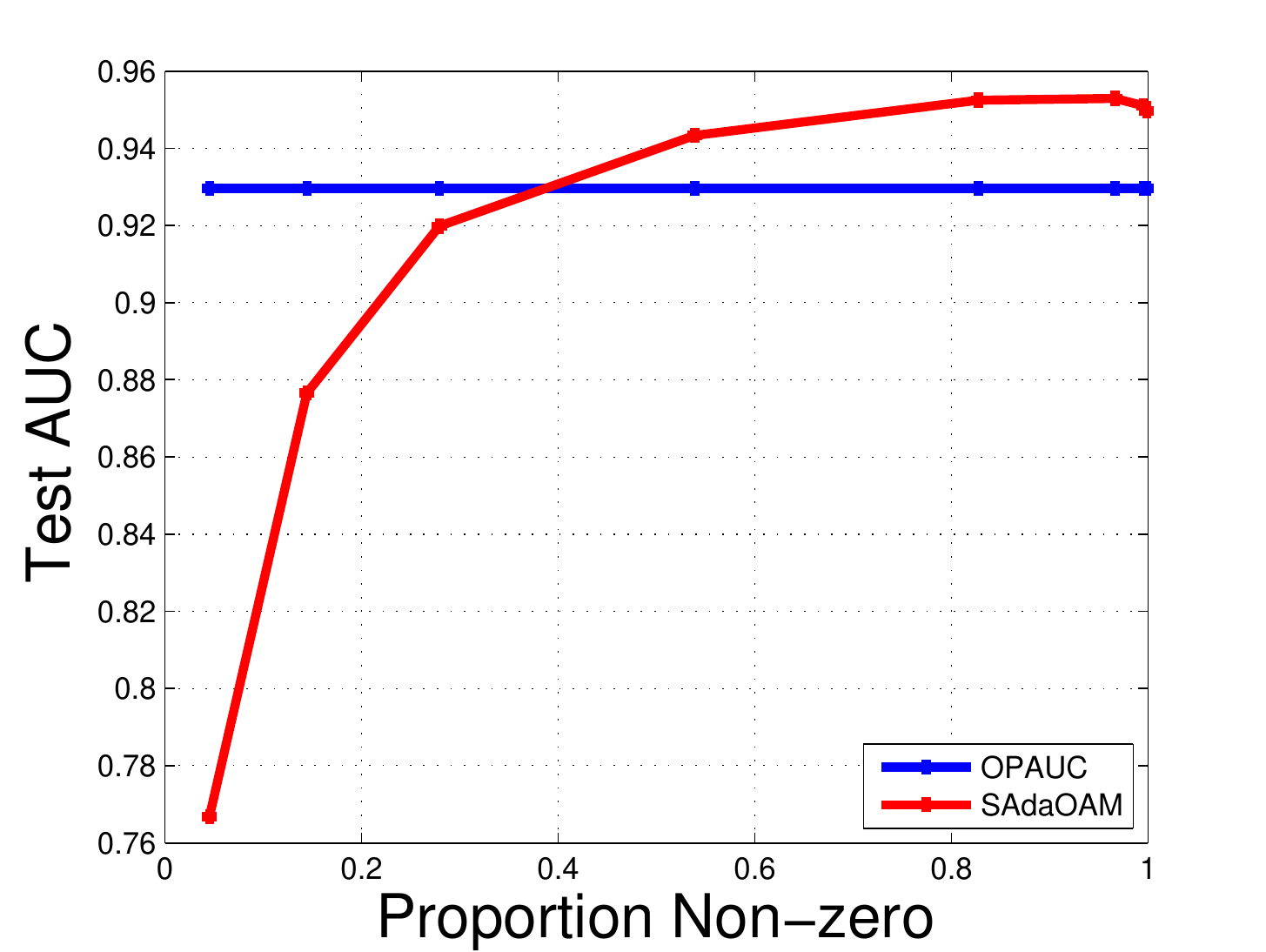}
	}\hspace{0.1in}
	\subfigure[farm ads]{
		\includegraphics[width=2.2in]{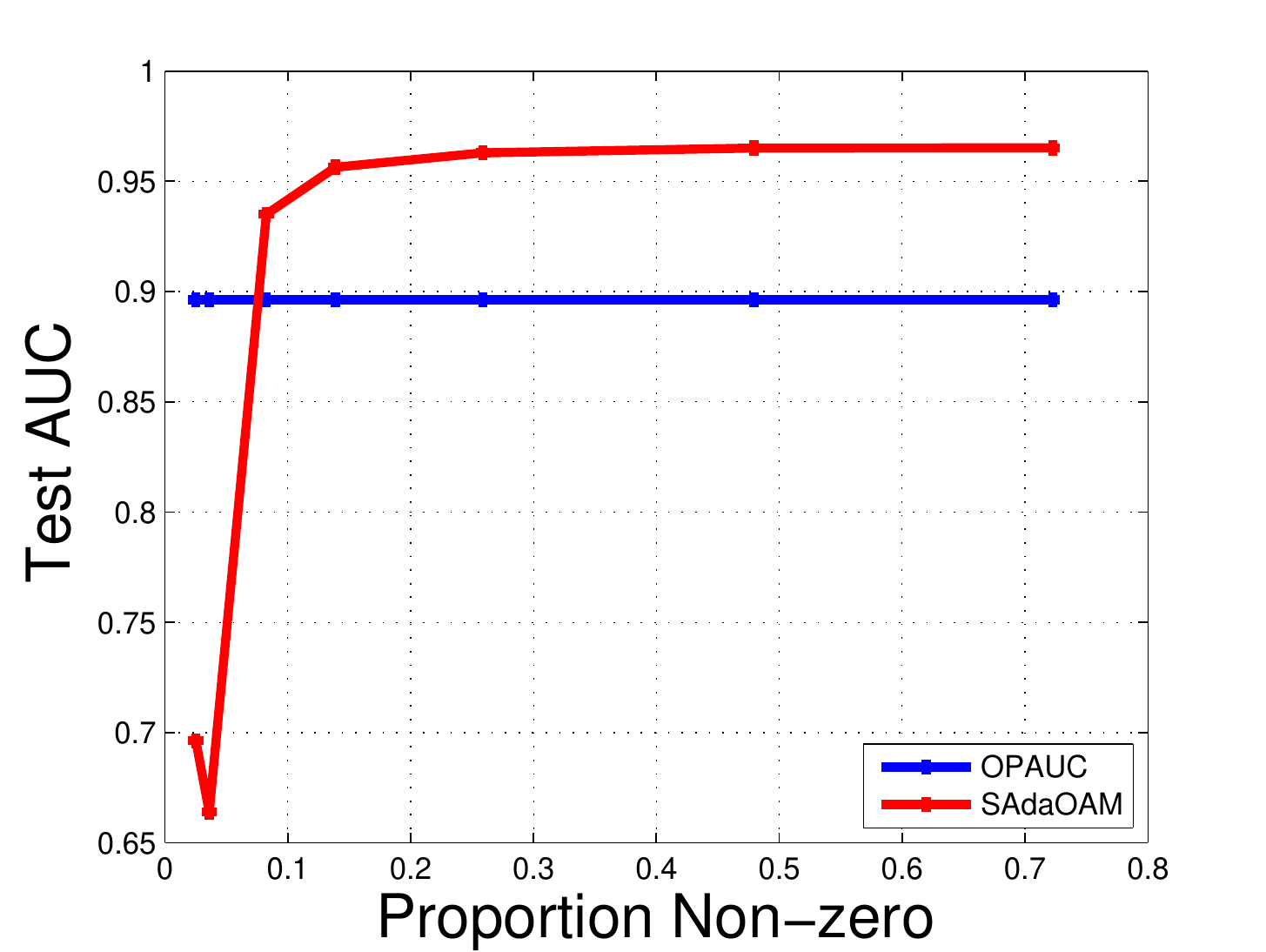}
	}\hspace{0.1in}
	\subfigure[Reuters]{
		\includegraphics[width=2.2in]{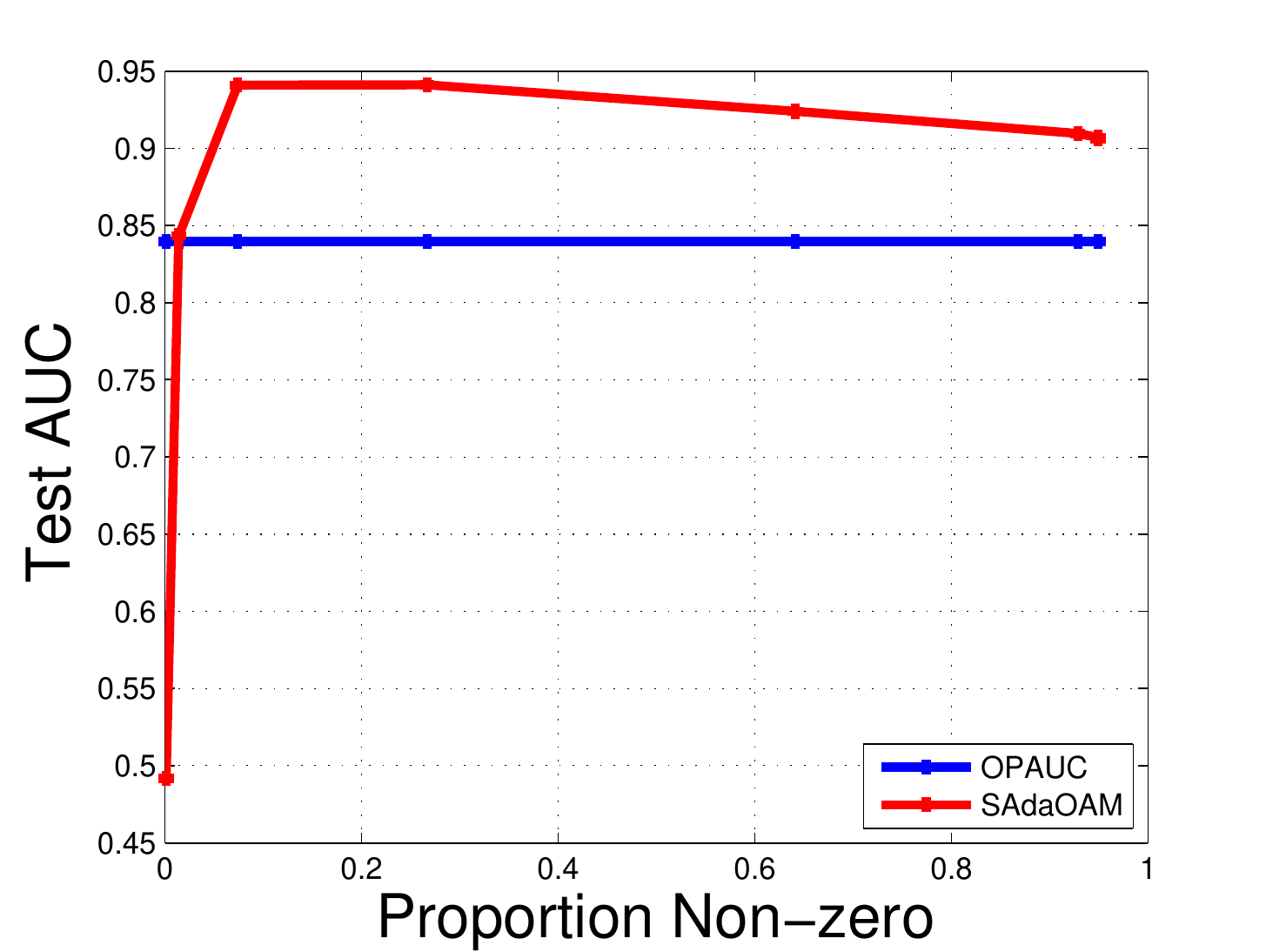}
	}\hspace{0.1in}
%	\subfigure[rcv1m]{
%		\includegraphics[width=2.4in]{fig/theta_rcv1m.pdf}
%	}
\vspace{-0.1in}
	\caption{Test AUC as a function of proportion of non-zeros in weight solution of SAdaOAM (OPAUC plotted for reference). In each plot, the horizontal blue line assigned the baseline performance of OPAUC generating nearly full dense weight models.} \label{fig:sparsity}
\end{figure*}

From this figure, we can observe that there are indeed tradeoffs between the level of sparsity and the AUC performance.  With high regularization parameter $\theta$, the SAdaOAM shows poor performance as expected since the weight vector is overly sparse and exhibits poor generalization. However, when the regularization parameter decreases, the learned weight becomes less sparse and eventually exceed the OPAUC's performance. More importantly, when the sparsity is small enough, the AUC performance of the SAdaOAM algorithm tend to become saturated for some datasets, such as farm ads, where further decreasing the sparsity of the model has very limited improvement on the AUC value. This implies that SAdaOAM can effectively learn a sparse model with small fraction of informative features, which can help remove those redundant features and reduce the testing time complexity.

\subsection{Application to Real World Online Anomaly Detection Task}

Online AUC maximization can be potentially applied to a wide range of applications. In this subsection, we showcase an application of the proposed algorithm namely, AdaOAM, for solving online anomaly detection
tasks. In particular, we begin with an introduction of the applications followed by a presentation of the empirical results. To be precise, consider the following four domains:
{\bf Webspam}: We apply the AdaOAM to detect malicious web pages using the ``webspam-u" dataset with unigram format from the subset used in the Pascal Large Scale Learning Challenge~\cite{DBLP:conf/colcom/WangIP12};
\textbf{Sensor Faults}: We apply the AdaOAM to identify sensor faults in buildings with the ``smartBuilding" dataset~\cite{DBLP:journals/tc/MichaelidesP09}, where the sensors monitor the concentration of the contaminant of interest (such as CO2) in different zones in a building;
\textbf{Malware App}: We apply the AdaOAM to detect mobile malware app with a ``malware" app permission dataset, which is built from the Android Malware Genome Project~\footnote{\url{http://www.malgenomeproject.org/}}~\cite{DBLP:conf/sp/ZhouJ12}. In our experiment, we adopt the dataset preprocessed by~\cite{DBLP:conf/ccs/PengGSLQPNM12} after data cleansing and duplication removal;
\textbf{Bioinformatics}: We apply the AdaOAM to solve a bioinformatics problem with the ``protein-h" dataset from the prediction task of the KDD Cup 2004~\cite{DBLP:journals/sigkdd/CaruanaJB04}. The aim is to predict which proteins are homologous to a native (query) sentence. Non-homologous sequences are labeled as anomalies.

Table~\ref{table:anomaly} summarizes the details of these datasets related to the above four different domains.

\begin{table}[ht]
	\vspace{-0.2in}
	\renewcommand*\arraystretch{1.2}
	\begin{center}
		\caption{\small Details of anomaly detection datasets.}\label{table:anomaly}
		\vspace{-0.1in}
		\begin{small}
			\begin{tabular}{|c|ccc|}         \hline
				datasets &    $\#$ inst        &$\#$ dim   &$T_{-} / T_{+}$ \\
				\hline\hline
				webspam-u         & 350,000   & 254 & 1.5397   \\
				smartBuilding     & 20,000    & 14  & 85.2069  \\
				malware           & 72,139    & 122 & 88.2809  \\
				protein-h         & 145,751   & 74  & 111.4622
				\\\hline
			\end{tabular}
		\end{small}
	\end{center}
	\renewcommand*\arraystretch{1.2}
	\vspace{-0.2in}
\end{table}

Table~\ref{table:AUC-anomaly} and Figure~\ref{fig/time_anomaly} have shown the performance and efficiency of the proposed algorithm for online anomaly detection task respectively. From Table~\ref{table:AUC-anomaly}, we observe that the proposed AdaOAM algorithm also outperforms other methods. Although OAM$_{seq}$ and OAM$_{gra}$ obtain comparably good results, their computational costs are very high, which are impractical for real-world learning tasks. Again, the AdaOAM proves its efficiency for real-world applications.

\begin{table*}[ht]
	\vspace{-0.1in}
	\renewcommand*\arraystretch{1.2}
	\begin{center}
		\caption{AUC performance evaluation (mean$\pm$std.) of AdaOAM versus the other online algorithms on anomaly detection datasets. $\bullet/\circ$ indicates that AdaOAM is significantly better/worse than the corresponding method (pairwise $t$-tests at 95\% significance level).}\label{table:AUC-anomaly}
		\vspace{-0.1in}
		\begin{scriptsize}
			\begin{tabular}{|c|c|c|c|c|c|c|c|c|c|}        \hline
				datasets       & AdaOAM  & OPAUC & OAM$_{seq}$ & OAM$_{gra}$ & online Uni-Log & online Uni-Exp  \\
				\hline
				webspam-u          &.964 	$\pm$ .005 &.959 	$\pm$ .006$\bullet$&.963 	$\pm$ .005         &.962 	$\pm$ .005         &.923 	$\pm$ .005$\bullet$&.920 	$\pm$ .006$\bullet$ \\
				smartBuilding      &.838 	$\pm$ .044 &.629 	$\pm$ .070$\bullet$&.629 	$\pm$ .069$\bullet$&.631 	$\pm$ .069$\bullet$&.749 	$\pm$ .020$\bullet$&.758 	$\pm$ .022$\bullet$ \\
				malware            &.967 	$\pm$ .008 &.919 	$\pm$ .008$\bullet$&.959 	$\pm$ .009$\bullet$&.953 	$\pm$ .009$\bullet$&.695 	$\pm$ .009$\bullet$&.765 	$\pm$ .009$\bullet$ \\
				protein-h          &.972 	$\pm$ .004 &.958    $\pm$ .005$\bullet$&.970    $\pm$ .007         &.968    $\pm$ .007$\bullet$&.890    $\pm$ .009$\bullet$&.915    $\pm$ .009$\bullet$ \\\hline
				\multicolumn{2}{|c|}{win/tie/loss}       &4/0/0                     &2/2/0                     &3/1/0                     &4/0/0                     &4/0/0                 \\
				\hline
			\end{tabular}
		\end{scriptsize}
	\end{center}
	\renewcommand*\arraystretch{1.2}
	\vspace{-0.1in}
\end{table*}

\begin{figure*}[!htbp]
	\centering\includegraphics[width=6in]{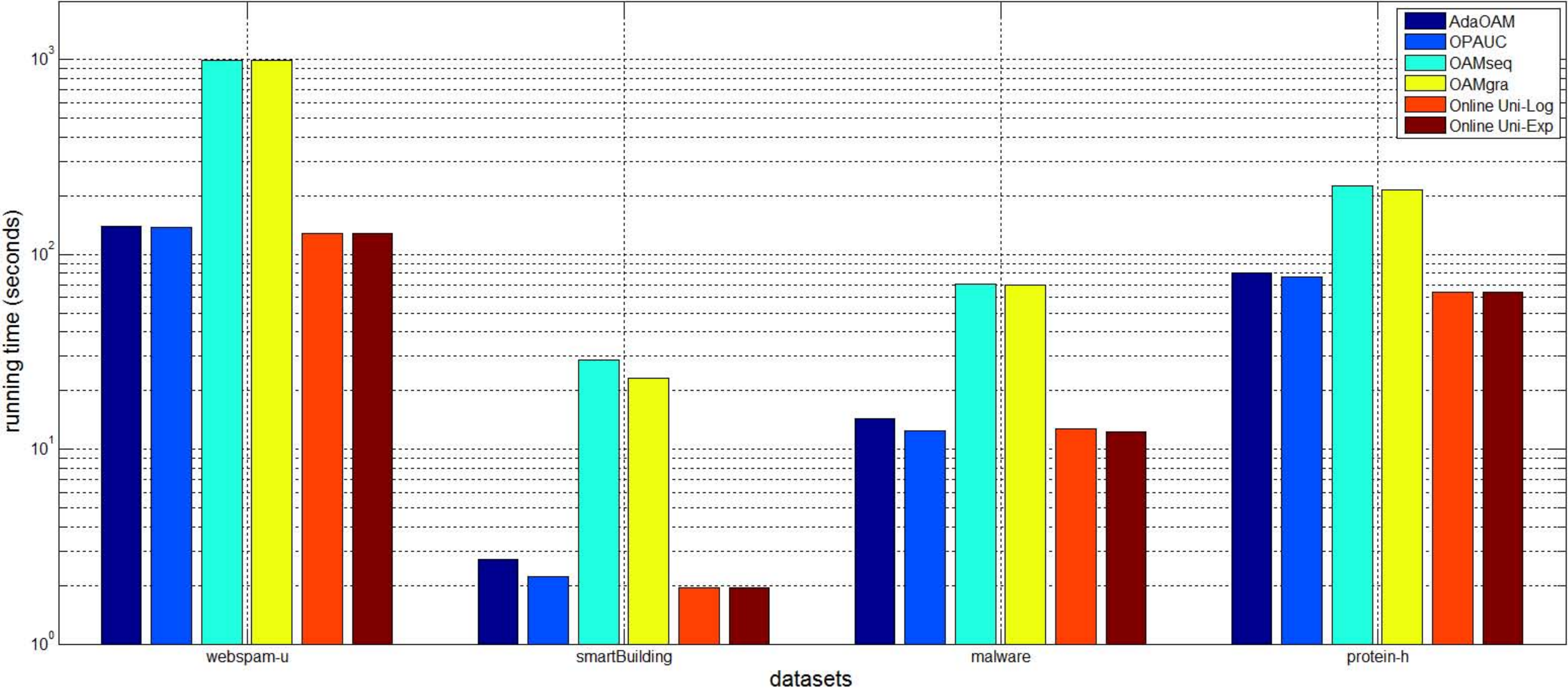}
	\vspace{-0.2in}
	\caption{Comparison of the running time (in seconds) of AdaOAM and other online learning algorithms on anomaly detection datasets. Notice that the \emph{y}-axis is in log-scale.}\label{fig/time_anomaly}\vspace{-0.2in}
\end{figure*}

\section{Conclusion}

In this paper, we have proposed two adaptive subgradient online AUC maximization approaches for handling both regular and high-dimensional sparse data, which considered the historical component-wise gradient information for more efficient and adaptive learning. Our proposed algorithms employ the second order information to speed up online AUC maximization, and are less sensitive to parameter setting than that of the simple SGD strategy. Theoretically, we have derived and analyzed the regret bound of the adaptive online AUC maximization approaches and verified that the proposed algorithms would achieve lower regret bound when handling both regular and high-dimensional sparse data. Empirically, we have also conducted extensive experimental studies with comparisons to a number of competing online AUC optimization algorithms on diverse types of data including many benchmark datasets, high-dimensional sparse datasets, and several real-world anomaly detection tasks. Overall, the obtained empirical results observations agree with our theoretical analyses and the results also verified the effectiveness and efficiency of the proposed AdaOAM and SAdaOAM algorithms.

\ifCLASSOPTIONcompsoc
  \section*{Acknowledgments}
  This research is partially supported by the Multi-plAtform Game Innovation Centre (MAGIC) in Nanyang Technological University. MAGIC is funded by the Interactive Digital Media Programme Office (IDMPO) hosted by the Media Development Authority of Singapore. IDMPO was established in 2006 under the mandate of the National Research Foundation to deepen Singapore's research capabilities in interactive digital media (IDM), fuel innovation and shape the future of media. In addition, the last author is grateful to the support provided by the Singapore MOE tier 1 research grant (C220/MSS14C003).
\else
  \section*{Acknowledgment}
\fi

\bibliographystyle{IEEEtran}
\bibliography{mybibfile}

\begin{thebibliography}{10}
\providecommand{\url}[1]{#1}
\csname url@samestyle\endcsname
\providecommand{\newblock}{\relax}
\providecommand{\bibinfo}[2]{#2}
\providecommand{\BIBentrySTDinterwordspacing}{\spaceskip=0pt\relax}
\providecommand{\BIBentryALTinterwordstretchfactor}{4}
\providecommand{\BIBentryALTinterwordspacing}{\spaceskip=\fontdimen2\font plus
\BIBentryALTinterwordstretchfactor\fontdimen3\font minus
  \fontdimen4\font\relax}
\providecommand{\BIBforeignlanguage}[2]{{%
\expandafter\ifx\csname l@#1\endcsname\relax
\typeout{** WARNING: IEEEtran.bst: No hyphenation pattern has been}%
\typeout{** loaded for the language `#1'. Using the pattern for}%
\typeout{** the default language instead.}%
\else
\language=\csname l@#1\endcsname
\fi
#2}}
\providecommand{\BIBdecl}{\relax}
\BIBdecl

\bibitem{Fawcett:2006:IRA:1159473.1159475}
T.~Fawcett, ``An introduction to roc analysis,'' \emph{Pattern Recogn. Lett.},
  vol.~27, pp. 861--874, 2006.

\bibitem{DBLP:conf/nips/CortesM03}
C.~Cortes and M.~Mohri, ``Auc optimization vs. error rate minimization,'' in
  \emph{NIPS}, 2003.

\bibitem{DBLP:conf/pkdd/CaldersJ07}
T.~Calders and S.~Jaroszewicz, ``Efficient auc optimization for
  classification,'' in \emph{PKDD}, 2007, pp. 42--53.

\bibitem{DBLP:conf/icml/Joachims05}
T.~Joachims, ``A support vector method for multivariate performance measures,''
  in \emph{ICML}, 2005, pp. 377--384.

\bibitem{DBLP:journals/jmlr/RudinS09}
C.~Rudin and R.~E. Schapire, ``Margin-based ranking and an equivalence between
  adaboost and rankboost,'' \emph{Journal of Machine Learning Research},
  vol.~10, pp. 2193--2232, 2009.

\bibitem{DBLP:conf/icml/ZhaoHJY11}
P.~Zhao, S.~C.~H. Hoi, R.~Jin, and T.~Yang, ``Online {AUC} maximization,'' in
  \emph{Proceedings of the 28th International Conference on Machine Learning,
  {ICML} 2011, Bellevue, Washington, USA, June 28 - July 2, 2011}, 2011, pp.
  233--240.

\bibitem{DBLP:conf/icml/GaoJZZ13}
W.~Gao, R.~Jin, S.~Zhu, and Z.-H. Zhou, ``One-pass auc optimization,'' in
  \emph{ICML (3)}, 2013, pp. 906--914.

\bibitem{DBLP:conf/icml/KarS0K13}
P.~Kar, B.~K. Sriperumbudur, P.~Jain, and H.~Karnick, ``On the generalization
  ability of online learning algorithms for pairwise loss functions,'' in
  \emph{Proceedings of the 30th International Conference on Machine Learning,
  {ICML} 2013, Atlanta, GA, USA, 16-21 June 2013}, 2013, pp. 441--449.

\bibitem{gao2012consistency}
W.~Gao and Z.-H. Zhou, ``On the consistency of auc optimization,'' \emph{arXiv
  preprint arXiv:1208.0645}, 2012.

\bibitem{DBLP:journals/jmlr/DuchiHS11}
J.~C. Duchi, E.~Hazan, and Y.~Singer, ``Adaptive subgradient methods for online
  learning and stochastic optimization,'' \emph{Journal of Machine Learning
  Research}, vol.~12, pp. 2121--2159, 2011.

\bibitem{DBLP:books/daglib/0016248}
N.~Cesa-Bianchi and G.~Lugosi, \emph{Prediction, learning, and games}.\hskip
  1em plus 0.5em minus 0.4em\relax Cambridge University Press, 2006.

\bibitem{DBLP:journals/jmlr/CrammerDKSS06}
K.~Crammer, O.~Dekel, J.~Keshet, S.~Shalev-Shwartz, and Y.~Singer, ``Online
  passive-aggressive algorithms,'' \emph{Journal of Machine Learning Research},
  vol.~7, pp. 551--585, 2006.

\bibitem{DBLP:journals/jmlr/ZhaoHJ11}
P.~Zhao, S.~C.~H. Hoi, and R.~Jin, ``Double updating online learning,''
  \emph{Journal of Machine Learning Research}, vol.~12, pp. 1587--1615, 2011.

\bibitem{DBLP:journals/ml/HoiJZY13}
S.~C.~H. Hoi, R.~Jin, P.~Zhao, and T.~Yang, ``Online multiple kernel
  classification,'' \emph{Machine Learning}, vol.~90, no.~2, pp. 289--316,
  2013.

\bibitem{DBLP:journals/ai/ZhaoHWL14}
P.~Zhao, S.~C.~H. Hoi, J.~Wang, and B.~Li, ``Online transfer learning,''
  \emph{Artif. Intell.}, vol. 216, pp. 76--102, 2014.

\bibitem{rosenblatt1958perceptron}
F.~Rosenblatt, ``The perceptron: a probabilistic model for information storage
  and organization in the brain.'' \emph{Psychological review}, vol.~65, no.~6,
  p. 386, 1958.

\bibitem{DBLP:conf/icml/DredzeCP08}
M.~Dredze, K.~Crammer, and F.~Pereira, ``Confidence-weighted linear
  classification,'' in \emph{Machine Learning, Proceedings of the Twenty-Fifth
  International Conference {(ICML} 2008), Helsinki, Finland, June 5-9, 2008},
  2008, pp. 264--271.

\bibitem{DBLP:conf/nips/CrammerKD09}
K.~Crammer, A.~Kulesza, and M.~Dredze, ``Adaptive regularization of weight
  vectors,'' in \emph{Advances in Neural Information Processing Systems 22:
  23rd Annual Conference on Neural Information Processing Systems 2009.
  Proceedings of a meeting held 7-10 December 2009, Vancouver, British
  Columbia, Canada.}, 2009, pp. 414--422.

\bibitem{DBLP:conf/nips/OrabonaC10}
F.~Orabona and K.~Crammer, ``New adaptive algorithms for online
  classification,'' in \emph{Advances in Neural Information Processing Systems
  23: 24th Annual Conference on Neural Information Processing Systems 2010.
  Proceedings of a meeting held 6-9 December 2010, Vancouver, British Columbia,
  Canada.}, 2010, pp. 1840--1848.

\bibitem{DBLP:conf/icml/HoiWZ12}
J.~Wang, P.~Zhao, and S.~C.~H. Hoi, ``Exact soft confidence-weighted
  learning,'' in \emph{Proceedings of the 29th International Conference on
  Machine Learning, {ICML} 2012, Edinburgh, Scotland, UK, June 26 - July 1,
  2012}, 2012.

\bibitem{DBLP:journals/tkde/WangZH14}
------, ``Cost-sensitive online classification,'' \emph{{IEEE} Trans. Knowl.
  Data Eng.}, vol.~26, no.~10, pp. 2425--2438, 2014.

\bibitem{DBLP:conf/kdd/ZhaoH13}
P.~Zhao and S.~C.~H. Hoi, ``Cost-sensitive online active learning with
  application to malicious {URL} detection,'' in \emph{The 19th {ACM} {SIGKDD}
  International Conference on Knowledge Discovery and Data Mining, {KDD} 2013,
  Chicago, IL, USA, August 11-14, 2013}, 2013, pp. 919--927.

\bibitem{DBLP:conf/sdm/HoiZ13}
S.~C.~H. Hoi and P.~Zhao, ``Cost-sensitive double updating online learning and
  its application to online anomaly detection,'' in \emph{Proceedings of the
  13th {SIAM} International Conference on Data Mining, May 2-4, 2013. Austin,
  Texas, {USA.}}, 2013, pp. 207--215.

\bibitem{DBLP:journals/jmlr/LangfordLZ09}
J.~Langford, L.~Li, and T.~Zhang, ``Sparse online learning via truncated
  gradient,'' \emph{Journal of Machine Learning Research}, vol.~10, pp.
  777--801, 2009.

\bibitem{DBLP:journals/jmlr/DuchiS09}
J.~C. Duchi and Y.~Singer, ``Efficient online and batch learning using forward
  backward splitting,'' \emph{Journal of Machine Learning Research}, vol.~10,
  pp. 2899--2934, 2009.

\bibitem{DBLP:journals/jmlr/Xiao10}
L.~Xiao, ``Dual averaging methods for regularized stochastic learning and
  online optimization,'' \emph{Journal of Machine Learning Research}, vol.~11,
  pp. 2543--2596, 2010.

\bibitem{DBLP:journals/jmlr/Shalev-ShwartzT11}
S.~Shalev{-}Shwartz and A.~Tewari, ``Stochastic methods for
  \emph{l}\({}_{\mbox{1}}\)-regularized loss minimization,'' \emph{Journal of
  Machine Learning Research}, vol.~12, pp. 1865--1892, 2011.

\bibitem{DBLP:journals/mp/Nesterov09}
Y.~Nesterov, ``Primal-dual subgradient methods for convex problems,''
  \emph{Math. Program.}, vol. 120, no.~1, pp. 221--259, 2009.

\bibitem{DBLP:conf/icml/Zinkevich03}
M.~Zinkevich, ``Online convex programming and generalized infinitesimal
  gradient ascent,'' in \emph{Machine Learning, Proceedings of the Twentieth
  International Conference {(ICML} 2003), August 21-24, 2003, Washington, DC,
  {USA}}, 2003, pp. 928--936.

\bibitem{DBLP:journals/jsw/YounM07}
S.~Youn and D.~McLeod, ``Spam email classification using an adaptive
  ontology,'' \emph{{JSW}}, vol.~2, no.~3, pp. 43--55, 2007.

\bibitem{paultseng2008}
P.~Tseng, ``On accelerated proximal gradient methods for convex-concave
  optimization,'' \emph{Technical report, Department of Mathematics, University
  of Washington}, 2008.

\bibitem{DBLP:journals/jmlr/HoiWZ14}
S.~C.~H. Hoi, J.~Wang, and P.~Zhao, ``{LIBOL:} a library for online learning
  algorithms,'' \emph{Journal of Machine Learning Research}, vol.~15, no.~1,
  pp. 495--499, 2014.

\bibitem{DBLP:conf/icml/KotlowskiDH11}
W.~Kotlowski, K.~Dembczynski, and E.~H{\"{u}}llermeier, ``Bipartite ranking
  through minimization of univariate loss,'' in \emph{Proceedings of the 28th
  International Conference on Machine Learning, {ICML} 2011, Bellevue,
  Washington, USA, June 28 - July 2, 2011}, 2011, pp. 1113--1120.

\bibitem{DBLP:journals/pami/LiTZ13}
N.~Li, I.~W. Tsang, and Z.~Zhou, ``Efficient optimization of performance
  measures by classifier adaptation,'' \emph{{IEEE} Trans. Pattern Anal. Mach.
  Intell.}, vol.~35, no.~6, pp. 1370--1382, 2013.

\bibitem{DBLP:journals/jmlr/TsangKC05}
I.~W. Tsang, J.~T. Kwok, and P.~Cheung, ``Core vector machines: Fast {SVM}
  training on very large data sets,'' \emph{Journal of Machine Learning
  Research}, vol.~6, pp. 363--392, 2005.

\bibitem{Mesterharm11activelearning}
C.~Mesterharm and M.~J. Pazzani, ``Active learning using on-line algorithms,''
  in \emph{In ACM SIGKDD International Conference on Knowledge Discovery and
  Data Mining (KDD}, 2011.

\bibitem{DBLP:conf/colcom/WangIP12}
D.~Wang, D.~Irani, and C.~Pu, ``Evolutionary study of web spam: Webb spam
  corpus 2011 versus webb spam corpus 2006,'' in \emph{8th International
  Conference on Collaborative Computing: Networking, Applications and
  Worksharing, CollaborateCom 2012, Pittsburgh, PA, USA, October 14-17, 2012},
  2012, pp. 40--49.

\bibitem{DBLP:journals/tc/MichaelidesP09}
M.~P. Michaelides and C.~G. Panayiotou, ``{SNAP:} fault tolerant event location
  estimation in sensor networks using binary data,'' \emph{{IEEE} Trans.
  Computers}, vol.~58, no.~9, pp. 1185--1197, 2009.

\bibitem{DBLP:conf/sp/ZhouJ12}
Y.~Zhou and X.~Jiang, ``Dissecting android malware: Characterization and
  evolution,'' in \emph{IEEE Symposium on Security and Privacy}, 2012, pp.
  95--109.

\bibitem{DBLP:conf/ccs/PengGSLQPNM12}
H.~Peng, C.~S. Gates, B.~P. Sarma, N.~Li, Y.~Qi, R.~Potharaju, C.~Nita-Rotaru,
  and I.~Molloy, ``Using probabilistic generative models for ranking risks of
  android apps,'' in \emph{ACM Conference on Computer and Communications
  Security}, 2012, pp. 241--252.

\bibitem{DBLP:journals/sigkdd/CaruanaJB04}
R.~Caruana, T.~Joachims, and L.~Backstrom, ``Kdd-cup 2004: results and
  analysis,'' \emph{{SIGKDD} Explorations}, vol.~6, no.~2, pp. 95--108, 2004.

\end{thebibliography}

\begin{IEEEbiography}[{\includegraphics[width=1in,height=1.25in,clip,keepaspectratio]{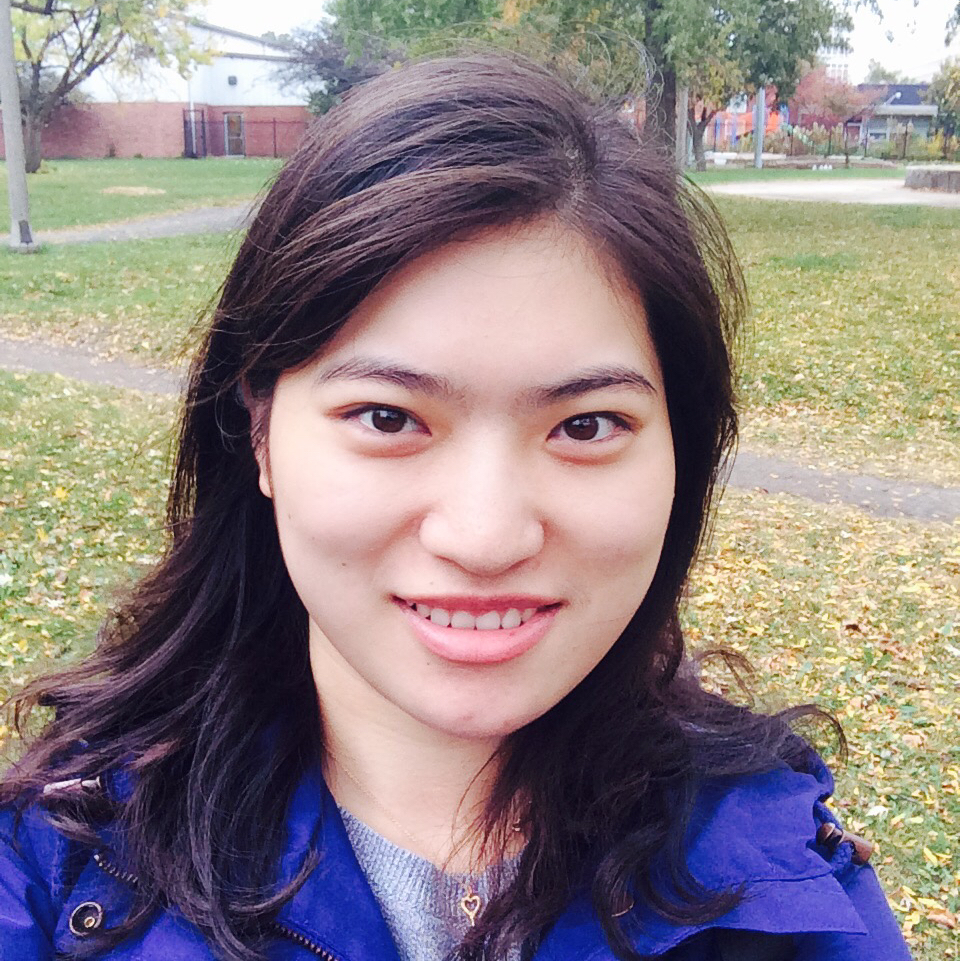}}]{Yi Ding} is currently a PhD student from the Department of Computer Science in the University of Chicago. She received her bachelor degree from Beijing Jiaotong University, Beijing, P.R. China, in 2012. Her research interests are statistical machine learning, optimization, and numerical analysis.
\end{IEEEbiography}
\begin{IEEEbiography}[{\includegraphics[width=1in,height=1.25in,clip,keepaspectratio]{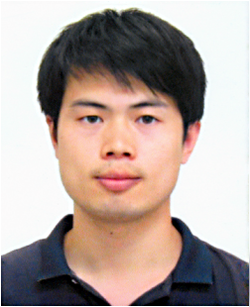}}]{Peilin Zhao} is currently a Research Scientist at Institute for Infocomm Research (I2R), A*STAR, Singapore.  He received his PhD from the School of Computer Engineering at the Nanyang Technological University, Singapore, in 2012 and his bachelor degree from Zhejiang University, Hangzhou, P.R. China, in 2008. His research interests are statistical machine learning, and and its applications to big data analytics, etc.
\end{IEEEbiography}
\begin{IEEEbiography}[{\includegraphics[width=1in,height=1.25in,clip,keepaspectratio]{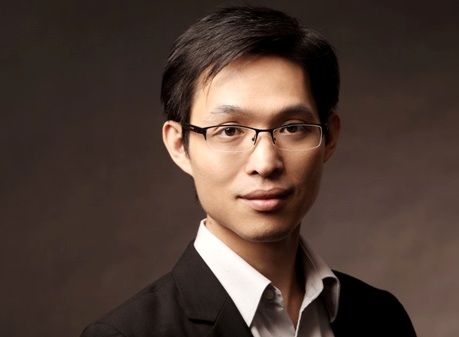}}]{Steven C.H. Hoi} is currently an Associate Professor of the School of Information Sytems, Singapore Management Unviersity, Singapore. Prior to joining SMU, he was Associate Professor with Nanyang Technological University, Singapore. He received his Bachelor degree from Tsinghua University, P.R. China, in 2002, and his Ph.D degree in computer science and engineering from The Chinese University of Hong Kong, in 2006. His research interests are machine learning and data mining and their applications to multimedia information retrieval (image and video retrieval), social media and web mining, and computational finance, etc, and he has published over 150 refereed papers in top conferences and journals in these related areas. He has served as Associate Editor-in-Chief for Neurocomputing Journal, general co-chair for ACM SIGMM Workshops on Social Media (WSM’09, WSM’10, WSM’11), program co-chair for the fourth Asian Conference on Machine Learning (ACML’12), book editor for “Social Media Modeling and Computing”, guest editor for ACM Transactions on Intelligent Systems and Technology (ACM TIST), technical PC member for many international conferences, and external reviewer for many top journals and worldwide funding agencies, including NSF in US and RGC in Hong Kong.
\end{IEEEbiography}
\begin{IEEEbiography}[{\includegraphics[width=1in,height=1.25in,clip,keepaspectratio]{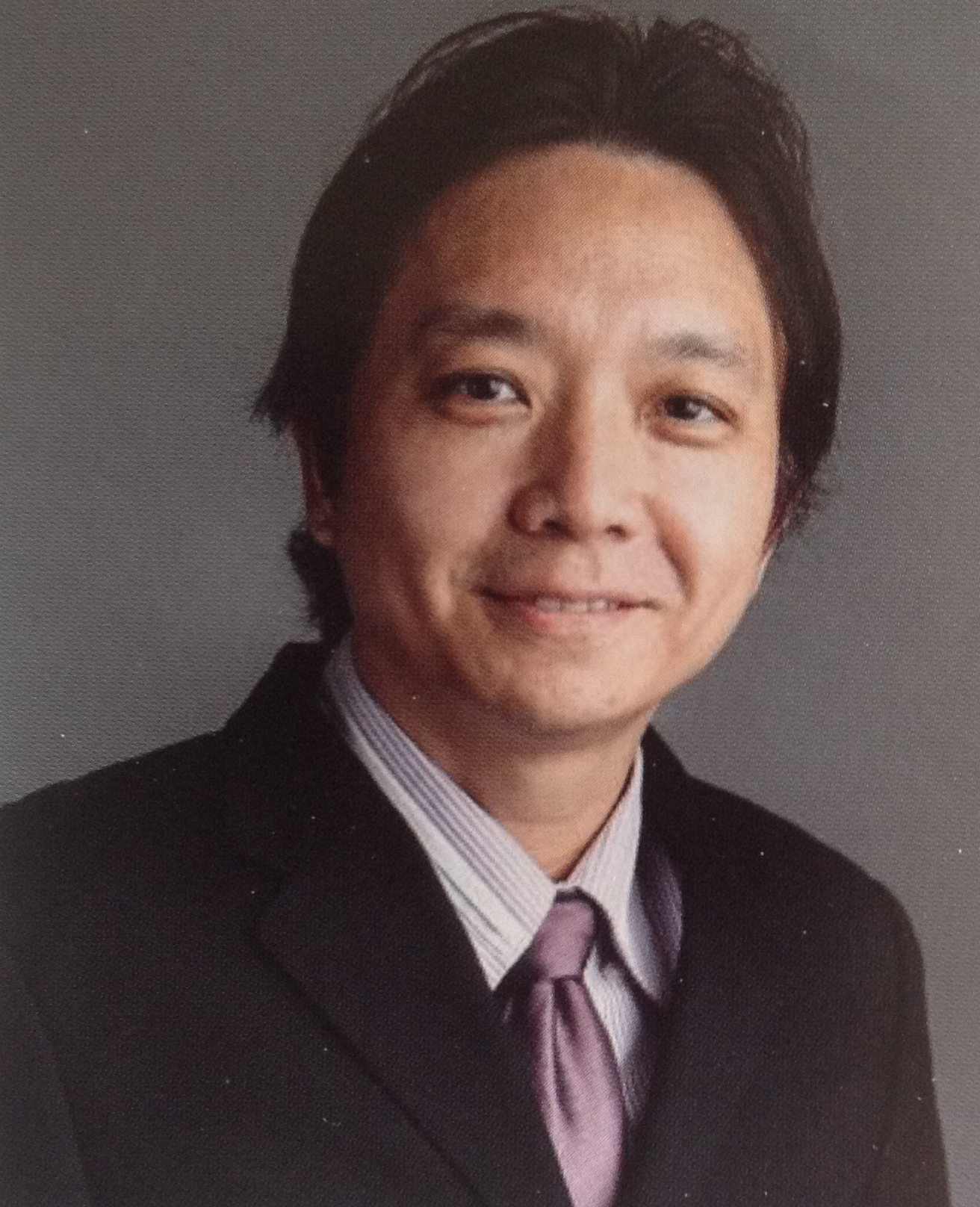}}]{Yew-Soon Ong} is Director of the A*Star SIMTECH-NTU Joint Lab on Complex Systems, Programme Principal Investigator of the Rolls-Royce@NTU Corporate Lab on Large Scale Data Analytics and a Professor of Computer Science in the School of Computer Engineering, Nanyang Technological University, Singapore. He was Director of the Centre for Computational Intelligence or Computational Intelligence Laboratory from 2008-2015. He received his Bachelors and Masters degrees in Electrical and Electronics Engineering from Nanyang Technological University and subsequently his PhD from University of Southampton, UK. His current research interests include computational intelligence spanning memetic computing, evolutionary optimization using approximation/surrogate/meta-models, complex design optimization, intelligent agents in game, and Big Data Analytics.
\end{IEEEbiography}
\end{document}